\documentclass{article}

\usepackage{amsmath,amsfonts,bm}

\def\eqref#1{equation~\ref{#1}}

\def\1{\bm{1}}

\DeclareMathAlphabet{\mathsfit}{\encodingdefault}{\sfdefault}{m}{sl}
\SetMathAlphabet{\mathsfit}{bold}{\encodingdefault}{\sfdefault}{bx}{n}

\newcommand{\KL}{D_{\mathrm{KL}}}

\DeclareMathOperator*{\argmin}{arg\,min}

\newcommand{\eg}{{\it e.g.}}
\newcommand{\ie}{{\it i.e.}}

\newcommand{\wrt}{{\it w.r.t. }}
\newcommand{\iid}{{\it i.i.d. }}

\newcommand{\reals}{\mathbb{R}}
\newcommand{\complex}{\mathbb{C}}

\newcommand{\sgn}{\text{sgn}}
\newcommand{\bv}{\bigg|}
\newcommand{\divg}{\text{div}}

\newcommand{\expect}{\mathbb{E}}
\newcommand{\prob}{\mathbf{Prob}}

\PassOptionsToPackage{numbers,sort}{natbib}
\usepackage[final]{neurips_2023}

\usepackage[utf8]{inputenc} 
\usepackage[T1]{fontenc}    
\usepackage{hyperref}       
\usepackage{url}            
\usepackage{booktabs}       
\usepackage{amsfonts}       
\usepackage{nicefrac}       
\usepackage{microtype}      
\usepackage{xcolor}         
\usepackage{bm}
\usepackage{graphicx}
\usepackage{multirow}
\usepackage{subfigure}
\usepackage{caption}
\usepackage{setspace}
\usepackage{amsthm,amsmath,amssymb}
\usepackage{mathrsfs}
\usepackage{geometry}
\usepackage{amsfonts}
\usepackage[utf8]{inputenc}
\usepackage{amssymb}
\usepackage{amsmath}
\usepackage{amsmath,amscd}
\everymath{\displaystyle}
\usepackage{indentfirst} 
\usepackage{enumerate}

\usepackage{algorithm, algorithmic}
\usepackage{chngcntr}
\usepackage{apptools}
\usepackage{wrapfig}
\AtAppendix{\counterwithin{thm}{section}}
\newtheorem{thm}{Theorem}
\newtheorem{rmk}{Remark}
\newtheorem{asp}{Assumption}
\newtheorem*{asp*}{Assumption}
\newtheorem{lemma}[thm]{Lemma}

\newtheorem{prop}[thm]{Proposition}
\newtheorem{example}{Example}
\theoremstyle{definition}
\newtheorem{definition}{Definition}[section]

\renewcommand{\eqref}[1]{(\ref{#1})}

\title{Particle-based Variational Inference with Generalized Wasserstein Gradient Flow}

\author{
    Ziheng Cheng\thanks{Contributed equally to this work.} \\
    School of Mathematical Sciences \\
    Peking University \\
    \texttt{alex-czh@stu.pku.edu.cn} \\
    \And
    $\textrm{Shiyue Zhang}^\ast$ \\
    School of Mathematical Sciences \\
    Peking University \\
    \texttt{zhangshiyue@stu.pku.edu.cn} \\
    \AND
    Longlin Yu \\
    School of Mathematical Sciences \\
    Peking University \\
    \texttt{llyu@pku.edu.cn} \\
    \And
    Cheng Zhang\thanks{Corresponding author.} \\
    School of Mathematical Sciences and Center for Statistical Science \\
    Peking University \\
    \texttt{chengzhang@math.pku.edu.cn} \\
}

\begin{document}

\maketitle

\begin{abstract}
Particle-based variational inference methods (ParVIs) such as Stein variational gradient descent (SVGD) update the particles based on the kernelized Wasserstein gradient flow for the Kullback-Leibler (KL) divergence.
However, the design of kernels is often non-trivial and can be restrictive for the flexibility of the method.
Recent works show that functional gradient flow approximations with quadratic form regularization terms can improve performance.
In this paper, we propose a ParVI framework, called generalized Wasserstein gradient descent (GWG), based on a generalized Wasserstein gradient flow of the KL divergence, which can be viewed as a functional gradient method with a broader class of regularizers induced by convex functions.
We show that GWG exhibits strong convergence guarantees.
We also provide an adaptive version that automatically chooses Wasserstein metric to accelerate convergence.
In experiments, we demonstrate the effectiveness and efficiency of the proposed framework on both simulated and real data problems.
\end{abstract}

\section{Introduction}
Bayesian inference is an important method in modern machine learning that provides powerful tools for modeling complex data and reasoning under uncertainty.
The core of Bayesian inference is to estimate the posterior distribution given the data.
As the posterior distribution is intractable in general, various approximation approaches have been developed, of which variational inference and Markov Chain Monte Carlo are two typical examples.
By reformulating the inference problem into an optimization problem, variational inference (VI) seeks an approximation within a certain family of distributions that minimizes the Kullback-Leibler (KL) divergence to the posterior \citep{Jordan1999AnIT, Wainwright08, Blei2016VariationalIA}.
Equipped with efficient optimization algorithms, VI allows fast training and easy scaling to large datasets.
However, the construction of approximating distributions can be restrictive which may lead to poor approximation.
Markov chain Monte Carlo (MCMC) methods simulate a Markov chain to directly draw samples from the posterior \citep{duane87, Neal2011-yo, SGLD, SGHMC}.
While being asymptotically unbiased, MCMC often takes a long time to converge, and it is also difficult to access the convergence.

Recently, particle based variational inference methods (ParVIs) have been proposed that tend to combine the best of both worlds \citep{Liu2016SVGD, Chen2018unified, Liu2019Understanding, Langosco2021, fan2022variational, alvarez-melis2022optimizing}.
In ParVIs, the approximating distribution is represented as a set of particles, which are iteratively updated by minimizing the KL divergence to the posterior.
This non-parametric nature significantly improves the flexibility of ParVIs upon classical VIs, and the interaction between particles also makes ParVIs more particle-efficient than MCMCs.
The most well-known particle based VI method is Stein Variational Gradient Descent (SVGD) \citep{Liu2016SVGD}.
It updates the particles by simulating the gradient flows of the KL divergence on a certain kernel related distribution space, where the gradient flows have a tractable form \citep{Liu2017SVGF, Chewi2020chi-squared}.
However, SVGD relies on the choice of an appropriate kernel function whose design is highly non-trivial and hence could limit the flexibility of the method.
Moreover, the required computation of the kernel matrix scales quadratically with the number of particles, which makes it costly to use a large number of particles.

Instead of using kernel induced functional gradients, many attempts have been made to expand the function class for gradient flow approximation \citep{Hu18, grathwohl2020, Langosco2021, dong2023particlebased}. By leveraging the more general neural networks as the function class together with more general regularizers, these approaches have shown improved performance over vanilla SVGD while not requiring expensive kernel computation.
However, these methods only use quadratic form regularizers where either the Wasserstein gradient flow or its preconditioned variant is recovered.

In this work, we propose a ParVI method based on a general formulation of minimizing movement scheme in Wasserstein space, which corresponds to a generalized Wasserstein gradient flow of KL divergence.
Using Legendre-Fenchel transformation, our method can also be viewed as a functional gradient method with a more general class of regularizers which include the previously used quadratic forms as special cases.
We provide a theoretical convergence guarantee of ParVIs with neural-net-estimated vector field for generalized Wasserstein gradient flow, which to the best of our knowledge, has not been established yet.
Perhaps surprisingly, our results show that assuming reasonably accurate vector field estimates, the iteration complexity of ParVIs matches the traditional Langevin Monte Carlo under weaker assumptions on the target distribution.
As an extension, we also propose an algorithm that can adaptively adjust the Wasserstein metric to accelerate convergence.
Extensive numerical experiments on both simulated and real data sets are conducted to demonstrate the efficiency of our method over existing ones.

\section{Background}
\paragraph{Notations.}Throughout this paper, we use $x$ to denote particle samples in $\reals^d$. Let $\mathcal{P}(\reals^d)$ denote all the probability distributions on $\reals^d$ that are absolute continuous with respect to the Lebesgue measure. We do not distinguish a probabilistic measure with its density function. For $x\in \reals^d$ and $p>1$, $\|x\|_p:=(|x_1|^p+\cdots+|x_d|^p)^{1/p}$ stands for the $\ell_p$-norm. The H{\"o}lder conjugate of $p$ is denoted by $q:=p/(p-1)$. Notation $g^*(\cdot)$ denotes the Legendre transform of a convex function $g(\cdot)$ on $\reals^d$.

\subsection{Particle-based Variational Inference}
Let $\pi\in \mathcal{P}(\reals^d)$ be the target distribution we wish to sample from. We can cast the problem of sampling as an optimization problem: to construct a distribution $\mu^*$ that minimizes the KL divergence
\begin{equation}
    \mu^*:=\arg\min_{\mu\in \mathcal{P}'} \KL(\mu\|\pi), 
\end{equation}
where $\mathcal{P}'\subseteq\mathcal{P}(\reals^d)$ is the variational family. 
Particle-based variational inference methods (ParVIs) is a class of VI methods where $\mathcal{P}'$ is represented as a set of particles. Assume the current particle distribution is $\mu$, then it holds that
\begin{equation}\label{eq:dKL}
    \frac{d}{d\epsilon}\bv_{\epsilon=0} \KL((id+\epsilon v)_{\#}\mu\|\pi)=-\expect_{\mu} \langle \nabla\log\frac{\pi}{\mu}, v\rangle.
\end{equation}
ParVIs aim to find the optimal vector field $v$ that minimizes \eqref{eq:dKL} in certain function class. For example, SVGD \citep{Liu2016SVGD} restricts $v$ in the unit ball of an reproducing kernel Hilbert space (RKHS) which has a closed-form solution by kernel trick.
Meanwhile, \citet{Hu18, grathwohl2020, Langosco2021, dong2023particlebased} consider a more general class of functions for $v$, i.e., neural networks, and minimize \eqref{eq:dKL} with some quadratic form regularizers.

\subsection{Minimizing Movement Scheme in Wasserstein Space}\label{sec:MMS}
Assume the cost function $c(\cdot, \cdot):\reals^d\times \reals^d \rightarrow \reals$ is continuous and bounded from below. Define the optimal transportation cost between two probabilistic measure $\mu, \nu$ as:
\begin{equation}\label{eq:wc_distance}
    W_c (\mu,\nu) := \inf_{\rho\in \Pi(\mu,\nu)} \int c(x, y)d\rho.
\end{equation}
Specifically, if $c(x,y)=\|x-y\|_p^p$ for some $p>1$, then we get the $p$-th power of Wasserstein-p distance $W_p(\mu,\nu)$. \citet{jordan1998variational} consider a minimizing movement scheme (MMS) under $W_2$ metric.
Given the current distribution $\mu_{kh}$, the distribution for next step is determined by
\begin{equation}
    \mu_{(k+1)h} := \argmin_{\mu\in \mathcal{P}_2(\reals^d)} \KL(\mu\|\pi) + \frac{1}{2h}W_2^2(\mu, \mu_{kh}).
\end{equation}
When the step size $h\to 0$, $\{\mu_{kh}\}_{k\geq 0}$ converges to the solution of the Fokker-Planck equation
\begin{equation}
    \partial_t\mu_t+\divg (\mu_t\nabla\log\pi) = \Delta \mu_t.
\end{equation}
Therefore, MMS corresponds to the deterministic dynamics
\begin{equation}\label{eq:wgf_l2}
    dx_t = v_tdt,\ v_t=\nabla\log\pi-\nabla\log\mu_t,
\end{equation}
where $\mu_t$ is the law of $x_t$. \eqref{eq:wgf_l2} is also known as the gradient flow of KL divergence under $W_2$ metric, which we refer to as $L_2$-GF \citep{ambrosio2005gradient}. Note that the Langevin dynamics $dx_t=\nabla \log\pi(x_t)dt+\sqrt{2}dB_t$ ($B_t$ is the Brownian motion) reproduces the same distribution curve $\{\mu_t\}_{t\geq 0}$ and thus also corresponds to the Wasserstein gradient flow \citep{jordan1998variational}.

\section{Proposed Methods}

\subsection{Minimizing Movement Scheme with A General Metric}
We start with generalizing the scope of the aforementioned MMS in Section \ref{sec:MMS} which is under $W_2$ metric.

\begin{definition}[Young function]
    A strictly convex function $g$ on $\reals^d$ is called Young function if $g(x)=g(-x), g(0)=0$, and for any fixed $z\in\reals^d\backslash \{0\}$, $hg(\frac{z}{h})\to\infty$, as $h\to 0$.
\end{definition}

\begin{thm}\label{thm:wgf}
    Given a continuously differentiable Young function $g$ and step size $h>0$, define cost function $c_h(x,y)=g(\frac{x-y}{h})h$. Suppose that $\pi, \mu_{kh}\in\mathcal{P}_{c_h}(\reals^d):=\{\mu\in \mathcal{P}(\reals^d): \expect_{\mu}[g(\frac{2x}{h})] < \infty \}$.
    Under some mild conditions of $g$ (see details in Proposition \ref{prop:metric}), $\mathcal{P}_{c_h}(\reals^d)$ is a Wasserstein space equipped with Wasserstein distance. Consider MMS under transportation cost $W_{c_h}$:
    \begin{equation}
        \mu_{(k+1)h} := \argmin_{\mu\in \mathcal{P}_{c_h}(\reals^d)} \KL(\mu\|\pi) + W_{c_h}(\mu, \mu_{kh}).
    \end{equation}
    Denote the optimal transportation map under $W_{c_h}$ from $\mu_{(k+1)h}$ to $\mu_{kh}$ by $T_k(\cdot)$. Then we have
    \begin{equation}
        \frac{T_k(x) - x}{h} = -\nabla g^*\left(\nabla\log\pi(x)-\nabla\log\mu_{(k+1)h}(x)\right).
    \end{equation}
\end{thm}
 
Please refer to Appendix \ref{app_sec:mms} for full statements and proofs. Informally, $\mu_{(k+1)h}\approx \mu_{kh}$ for small step size $h$ \citep{santambrogio2017euclidean}. Further note that $\frac{T_k(x)-x}{h}$ is the optimal velocity field associated with the transport from $\mu_{(k+1)h}$ to $\mu_{kh}$ (and not vice versa). If step size $h\to 0$, then following \cite{jordan1998variational}, we can recover the dynamics in continuous time: 
\begin{equation}\label{eq:wgf_general}
    dx_t = v_tdt,\ v_t=\nabla g^*(\nabla\log\pi-\nabla\log\mu_t).
\end{equation}
We call \eqref{eq:wgf_general} the generalized Wasserstein gradient (GWG) flow. If we set $g(\cdot)=\frac{1}{2}\|\cdot\|_2^2$ or any positive definite quadratic form $g(\cdot)=\frac{1}{2}\|\cdot\|_H^2$, then \eqref{eq:wgf_general} reduces to $L_2$-GF \eqref{eq:wgf_l2} or its preconditioned version \citep{dong2023particlebased} respectively.

\subsection{Faster Descent of KL Divergence}
It turns out that we can leverage the general formulation \eqref{eq:wgf_general} to explore the underlying structure of different probability spaces and further utilize this geometric structure to accelerate sampling. More specifically, we consider $g(\cdot)=\frac{1}{p}\|\cdot\|_p^p$ for some $p>1$ and then $g^*(\cdot)=\frac{1}{q}\|\cdot\|_q^q$. Note that if the particles move along the vector field $v_t=\nabla g^*(\nabla\log\frac{\pi}{\mu_t})$, then the descent rate of $\KL(\mu_t\|\pi)$ is
\begin{equation}
    \partial_t \KL(\mu_t\|\pi) = -\expect_{\mu_t} \bv\bv \nabla\log\frac{\pi}{\mu_t}\bv\bv_q^q.
\end{equation}
If we choose $q$ such that $\expect_{\mu_t} \bv\bv \nabla\log\frac{\pi}{\mu_t}\bv\bv_q^q$ is large, then $\KL(\mu_t\|\pi)$ decreases faster and the sampling process can be accelerated. We use the following example to further illustrate our idea. Please refer to Appendix \ref{app_sec:dme} for detailed analysis.

\begin{example}\label{eg:Lq_mog}
Let $\pi=\frac{1}{2}\mathcal{N}(-m,1) + \frac{1}{2}\mathcal{N}(m,1)$ and $\mu=\frac{3}{4}\mathcal{N}(-m,1) + \frac{1}{4}\mathcal{N}(m,1)$. Then for any $m\geq \frac{1}{80}, q\geq 1$, the following holds:
\begin{equation}
    \frac{0.08}{qm}(\frac{m}{3})^q\exp(-\frac{m^2}{2}) \leq \expect_\mu \bv\bv \nabla\log\frac{\pi}{\mu} \bv\bv_q^q \leq \frac{0.2}{qm}(4m)^q\exp(-\frac{m^2}{2}).
\end{equation}
However, the KL divergence between $\pi$ and $\mu$ is large: $\KL(\mu\|\pi) \geq \frac{1}{10\sqrt{2}}$.
\end{example}

Suppose the target distribution is $\pi$ and we run ParVI with current particle distribution $\mu$. We can expect that, if simply using $L_2$ regularization, \ie, $q=2$, then for very large $m$, the score divergence is small and thus the decay of KL divergence is extremely slow. However, $\KL(\mu\|\pi)$ is still large, indicating that it would take a long time for the dynamics to converge to the target. But if we set $q$ much larger, then the derivative of KL divergence would get larger and the convergence can be accelerated.

\subsection{Algorithm}
The forward-Euler discretization of the dynamics \eqref{eq:wgf_general} is 
\begin{equation}\label{eq:discrete_dynamics}
    x_{(k+1)h} = x_{kh} + \nabla g^*\left(\nabla \log\frac{\pi}{\mu_{kh}}(x_{kh})\right)h.
\end{equation}
However, since the score of current particle distribution $\mu_{kh}$ is generally unknown, we need a method to efficiently estimate the GWG direction $\nabla g^*(\nabla \log\frac{\pi}{\mu_{kh}})$.
Given the distribution of current particles $\mu$, by the definition of convex conjugate, we have
\begin{equation*}
    \nabla g^*(\nabla \log\frac{\pi}{\mu})=\arg\max_{v} \expect_{\mu} [\langle \nabla \log\frac{\pi}{\mu}, v\rangle - g(v) ].
\end{equation*}
If we parameterize $v$ as a neural network $f_w$ with $w\in \mathcal{W}$, then we can maximize the following objective with respect to $w$:
\begin{equation}\label{eq:gsm_pop}
    \begin{aligned}
        \mathcal{L}(w)
        :&= \expect_{\mu} [\langle \nabla\log\frac{\pi}{\mu}, f_w\rangle - g(f_w)] \\
        &= \expect_{\mu} [(\nabla\log\pi)^Tf_w + \nabla\cdot f_w - g(f_w)]
    \end{aligned}
\end{equation}
Here the second equation is by \textit{Stein's identity} (we assume $\mu$ vanishes at infinity).
This way, the gradient of $\mathcal{L}(w)$ can be estimated via Monte Carlo methods given the current particles.
We summarize the procedures in Algorithm \ref{alg:parvi}.

\begin{algorithm}[t]
\caption{GWG: Generalized Wasserstein Gradient Flow}
    \label{alg:parvi}
    \begin{algorithmic}
        \REQUIRE{Unnormalized target distribution $\pi$, initial particles $\{x_0^i\}_{i=1}^{n}$, initial parameter $w_0$, iteration number $N, N'$, particle step size $h$, parameter step size $\eta$}
        \FOR{$k=0, \cdots, N-1$}
            \STATE{Assign $w_k^0 = w_k$}
            \FOR{$t=0, \cdots, N'-1$}
                \STATE{Compute 
                    \begin{equation}\label{eq:gsm}
                        \widehat{\mathcal{L}}(w) = \frac{1}{n}\sum_{i=1}^{n} \nabla \log{\pi(x_k^i)}^T f_w(x_k^i) + \nabla \cdot f_w(x_k^i) - g(f_w(x_k^i))
                    \end{equation}
                }
                \STATE{Update $w_k^{t+1} = w_k^{t} + \eta \nabla_w \widehat{\mathcal{L}}(w_k^t)$}
            \ENDFOR 
            \STATE{Update $w_{k+1} = w_k^{N'}$}
            \STATE{Update particles $x_{k+1}^{i} = x_{k}^{i} + hf_{w_{k+1}}(x_{k}^{i})$ for $i=1, \cdots, n$}
        \ENDFOR
        \RETURN{Particles $\{x_N^i\}_{i=1}^{n}$}
    \end{algorithmic}
\end{algorithm}

The exact computation of the divergence term $\nabla_x \cdot f_w(x)$ needs $\mathcal{O}(d)$ times back-propagation, where $d$ is the dimension of $x$. In order to reduce computation cost, we refer to Hutchinson’s estimator \citep{hutchinson1989stochastic}, \ie,
\begin{equation}\label{eq:gssm}
     \frac{1}{n}\sum_{i=1}^{n}\nabla \cdot f_w(x_k^i) \approx \frac{1}{n}\sum_{i=1}^{n} \xi_i^T\nabla (f_w(x_k^i)\cdot \xi_i),
\end{equation}
where $\xi_i\in \reals^d$ are independent random vectors satisfying $\expect \xi_i\xi_i^T = I_d$. This is still an unbiased estimator but only needs $\mathcal{O}(1)$ times back-propagation.

\section{Convergence Analysis without Isoperimetry}
\label{sec:convergence_parvi}
In this section, we state our main theoretical results of Algorithm \ref{alg:parvi}. Consider the discrete dynamics:
\begin{equation}
    X_{(k+1)h} = X_{kh} + v_k(X_{kh})h,
\end{equation}
where $v_k$ is the neural-net-estimated GWG at time $kh$. Define the interpolation process 
\begin{equation}
    X_t = X_{kh} + (t-kh) v_k(X_{kh}),\ \text{for}\ t\in [kh, (k+1)h],
\end{equation} 
and let $\mu_t$ denote the law of $X_t$. Note that here we do not assume isoperimetry of target distribution $\pi$ (\eg, \textit{log-Sobolev inequality}) and hence establish the convergence of dynamics in terms of score divergence, following the framework of non-log-concave sampling \citep{Balasubramanian2022TowardsAT}. 

We first make some basic assumptions. For simplicity, only two types of Young function $g^*$ are considered here, which are also the most common choices. 
\begin{asp}\label{asp:error_lp}
    $g^*(\cdot) = \frac{1}{q}\|\cdot\|_q^q$ for some $q>1$. And for any $k$, $\expect_{\mu_{kh}} \bv\bv v_k - \nabla g^*(\nabla \log{\frac{\pi}{\mu_{kh}}}) \bv\bv_p^p \leq \varepsilon_k$.
\end{asp}
\begin{asp}\label{asp:error_l2}
    $g^*(\cdot)$ is $\alpha$-strongly convex and $\beta$-smooth. Define $\kappa:=\frac{\beta}{\alpha}$. And for any $k$, $\expect_{\mu_{kh}} \bv\bv v_k - \nabla g^*(\nabla \log{\frac{\pi}{\mu_{kh}}})\bv\bv_2^2 \leq \varepsilon_k$. 
\end{asp}
The two assumptions above ensure the estimation accuracy of neural nets. Note that the preconditioned quadratic form in \cite{dong2023particlebased} is included in Assumption \ref{asp:error_l2}. Although the estimation error is not exactly the training objective used in Algorithm \ref{alg:parvi},the following proposition shows the equivalence between them in some sense. 

\begin{prop}\label{prop:a1a2}
    Suppose $g(\cdot)=\frac{1}{p}\|\cdot\|_p^p$ for some $p>1$. Given current particle distribution $\mu$, we can define the training loss $\mathcal{L}_{\text{train}}(v):=\expect_{\mu} [\langle \nabla\log\frac{\pi}{\mu}, v\rangle - g(v)] $. The maximizer is $v^*=\nabla g^*(\nabla\log\frac{\pi}{\mu})$ and the maximum value is $\mathcal{L}_{\text{train}}^*:=\mathcal{L}_{\text{train}}(v^*)<\infty$. For any arbitrarily small $\varepsilon_1>0$, there exists $\varepsilon_2:=\varepsilon_2(\varepsilon_1, p)<\infty$, such that 
    \begin{equation*}
        \expect_{\mu} \bv\bv v-\nabla g^*(\nabla\log\frac{\pi}{\mu})\bv\bv_p^p \leq \varepsilon_1 \mathcal{L}_{\text{train}}^* + \varepsilon_2 [\mathcal{L}_{\text{train}}^*-\mathcal{L}_{\text{train}}(v)].
    \end{equation*}
    Besides, if $p\geq 2$, $\varepsilon_1$ can be $0$ while $\varepsilon_2$ is still finite.
\end{prop}
Similar results also hold if $g$ satisfies Assumption \ref{asp:error_l2} since it is equivalent to the case when $p=2$.
Additionally, we expect some properties of the estimated vector fields.
\begin{asp}[Smoothness of neural nets]\label{asp:smooth}
For any $k$, $v_k(\cdot)$ is twice differentiable. For any $p>1$, $G_p:= \sup_{x,y} \frac{\|v_k(x)-v_k(y)\|_p}{\|x-y\|_p} < \infty$, $M_p:= \sup_{x,z} \lim_{\delta\to 0^+}\frac{\|\nabla v_k(x+\delta z)-\nabla v_k(x)\|_{op}}{\delta\|z\|_p} < \infty$.
\end{asp}
Note that here we do not assume the smoothness of potential $\log\pi$ explicitly. But informally, $G_p$ and $M_p$ correspond to the Lipschitz constant of the gradient and the Hessian of $\log\pi$, respectively. 

Let $\bar{\mu}_{Nh}:=\frac{1}{Nh}\int_{0}^{Nh} \mu_tdt$ and $K_0:=\KL(\mu_0\|\pi)$. Now we present our main results.

\begin{thm}[Full version see Theorem \ref{thm:lp}]\label{thm:rate_lp}
    Under Assumption \ref{asp:error_lp}, \ref{asp:smooth}, the following bound holds with proper step size $h$:
    \begin{equation}
         \expect_{\bar{\mu}_{Nh}} \bv\bv \nabla\log\frac{\pi}{\bar{\mu}_{Nh}}\bv\bv_q^q = \Tilde{\mathcal{O}}\left((\frac{M_pK_0d}{N})^{\frac{q}{q+1}} + \frac{G_2K_0d}{N} + \frac{\sum_{k=0}^{N-1}\varepsilon_k}{N}\right).
    \end{equation}
    Here $\Tilde{\mathcal{O}}(\cdot)$ hides all the constant factors that only depend on $q$.
\end{thm}

\begin{thm}[Full version see Theorem \ref{thm:l2}]\label{thm:rate_l2}
    Under Assumption \ref{asp:error_l2}, \ref{asp:smooth} with $\alpha=1$, the following bound holds with proper step size $h$:
    \begin{equation}
         \expect_{\bar{\mu}_{Nh}} \bv\bv \nabla\log\frac{\pi}{\bar{\mu}_{Nh}}\bv\bv_2^2 = \mathcal{O}\left((\frac{\kappa M_2K_0d}{N})^{\frac{2}{3}} + \frac{G_2K_0(d+\kappa)}{N} + \frac{\sum_{k=0}^{N-1}\varepsilon_k}{N}\right).
    \end{equation}
\end{thm}

The proofs in this section are deferred to Appendix \ref{sec:pf_main}. To interpret our results, suppose $g^*(\cdot)=\frac{1}{q}\|\cdot\|_q^q$ and $\epsilon\lesssim (\frac{M_p}{G_2})^q$. If the neural net $v_k(\cdot)$ can approximate $\nabla g^*(\nabla \log{\frac{\pi}{\mu_{kh}}})$ accurately (\ie, $\varepsilon_k\lesssim \epsilon$), then to obtain a probabilistic measure $\mu$ such that $\expect_{\mu} \bv\bv \nabla\log\frac{\pi}{\mu}\bv\bv_q^q \lesssim \epsilon$, the iteration complexity is $\Tilde{\mathcal{O}}(M_pK_0d\epsilon^{-(1+\frac{1}{q})})$. If we further let $q=p=2$, the complexity is $\mathcal{O}(K_0d\epsilon^{-\frac{3}{2}})$, which matches the complexity of Langevin Monte Carlo (LMC) under the Hessian smoothness and the growth order assumption \citep{Balasubramanian2022TowardsAT}. However, noticing that Assumption \ref{asp:smooth} is similar to Hessian smoothness informally, we can obtain this rate without additional assumption on target distribution. This suggests the potential benefits of particle-based methods.

In addition, our formulation allows a wider range of choices of Young function, including $\|\cdot\|_p^p$ and the preconditioned quadratic form \citep{dong2023particlebased}. This provides wider options of convergence metrics. We refer the readers to Appendix \ref{app_subsec:kl} for more discussions.

\section{Extensions: Adaptive Generalized Wasserstein Gradient Flow}

The GWG framework also allows adaption of the Young function $g$, instead of a fixed one. Similar ideas are also presented in \citet{wang2018variational}. In this section, we consider a special Young function class $\left\{ \frac{1}{p}\|\cdot\|_p^p: p> 1\right\}$ and propose a procedure that adaptively chooses $p$ to accelerate convergence.
Consider the continuous time dynamics $dx_t = f_t(x_t)dt$ and denote the distribution of particles at time $t$ as $\mu_t$, we have the following proposition.

\begin{prop}\label{prop:ada}
    For $g(\cdot)=\frac{1}{p}\|\cdot\|_p^p$, the derivative of KL divergence has an upper bound:
    \begin{equation}
        \partial_t \KL(\mu_t \| \pi) \leq -\frac{1}{p}\expect_{\mu_t} \bv\bv \nabla\log\frac{\pi}{\mu_t}\bv\bv_q^q + \frac{1}{p}\expect_{\mu_t} \bv\bv \nabla g^*(\nabla\log\frac{\pi}{\mu_t}) - f_t \bv\bv_p^p.
    \end{equation}
\end{prop}

The proof is in Appendix \ref{app_sec:prop_ada}. If the neural network $f_t$ can approximate the objective well, \ie,  $f_t\approx \nabla g^*(\nabla\log\frac{\pi}{\mu_t})$, then informally we can omit the second term and thus 
\begin{equation}\label{eq:A(p)}
    \partial_t \KL(\mu_t \| \pi) \lesssim -\frac{1}{p}\expect_{\mu_t} \bv\bv \nabla\log\frac{\pi}{\mu_t}\bv\bv_q^q \approx -\frac{1}{p}\expect_{\mu_t} \|f_t\|_p^p =: -A(p).
\end{equation}
In order to let KL divergence decrease faster, we can choose $p$ such that $A(p)$ is larger.
This leads to a simple adaptive procedure that updates $p$ by gradient ascent \wrt $A(p)$.
In practice, the adjustment of $p$ is delicate and would cause numerical instability if $p$  becomes excessively small or large. Therefore it is necessary to clip $p$ within a reasonable range.
We call this adaptive version of GWG, Ada-GWG.
The whole training procedure of Ada-GWG is shown in Algorithm \ref{alg:ada_parvi}. Note that \eqref{eq:ada_gsm} can be replaced with Hutchinson's estimator \eqref{eq:gssm} to improve computational efficiency as before.

\begin{algorithm}[t]
\caption{Ada-GWG: Adaptive Generalized Wasserstein Gradient Flow}
    \label{alg:ada_parvi}
    \begin{algorithmic}
        \REQUIRE{unnormalized target distribution $\pi$, initial particles $\{x_0^i\}_{i=1}^{n}$, initial parameter $w_0$, iteration number $N, N'$, step size $h, \eta, \Tilde{\eta}$, lower and upper bounds on $p$: $lb, ub$}
        \FOR{$k=0, \cdots, N-1$}
            \STATE{Assign $w_k^0 = w_k$}
            \FOR{$t=0, \cdots, N'-1$}
                \STATE{Compute 
                    \begin{equation}\label{eq:ada_gsm}
                        \widehat{\mathcal{L}}(w) = \frac{1}{n}\sum_{i=1}^{n} \nabla \log{\pi(x_k^i)}^T f_w(x_k^i) + \nabla \cdot f_w(x_k^i) - \frac{1}{p_k}\|f_w(x_k^i)\|_{p_k}^{p_k}
                    \end{equation}
                }
                \STATE{Update $w_k^{t+1} = w_k^{t} + \eta \nabla_w \widehat{\mathcal{L}}(w_k^t)$}
            \ENDFOR 
            \STATE{Update $w_{k+1} = w_k^{N'}$}
            \STATE{Compute $\widehat{A}(p_k)=\frac{1}{n} \sum_{i=1}^{n} \frac{1}{p_k}\|f_{w_{k+1}}(x_k^{i})\|_{p_k}^{p_k}$}
            \STATE{Update $p_{k+1} = \textbf{clip}(p_{k} + \Tilde{\eta}\nabla\widehat{A}(p_k), lb, ub)$}
            \STATE{Update particles $x_{k+1}^{i} = x_{k}^{i} + hf_{w_{k+1}}(x_{k}^{i})$ for $i=1, \cdots, n$}
            
        \ENDFOR
        \RETURN{Particles $\{x_N^i\}_{i=1}^{n}$}
    \end{algorithmic}
\end{algorithm}

\section{Numerical Experiments}

In this section, we compare GWG and Ada-GWG with other ParVI methods including SVGD \citep{Liu2016SVGD}, $L_2$-GF \citep{Langosco2021} and PFG \citep{dong2023particlebased} on both synthetic and real data problems.
In BNN experiments, we also test stochastic gradient Langevin dynamics (SGLD).
For Ada-GWG, the exponent $p$ is clipped between $1.1$ and $4.0$ unless otherwise specified.
Throughout this section, we choose $f_w$ to be a neural network with $2$ hidden layers and the initial particle distribution is $\mathcal{N}(0, I_d)$.
We refer the readers to Appendix \ref{sec:exp_setup} for more detailed setups of our experiments. The code is available at \url{https://github.com/Alexczh1/GWG}.

\subsection{Gaussian Mixture}

Our first example is on a multi-mode Gaussian mixture distribution.
Following \cite{dong2023particlebased}, we consider the 10-cluster Gaussian mixture where the variances of the mixture components are all 0.1. The number of particles is 1000.
Figure \ref{fig:gmm} shows the scatter plots of the sampled particles at different numbers of iterations.
We see that on this simple toy example, PFG performs similarly to the standard $L_2$-GF which does not involve the preconditioner, while Ada-GWG with the initial $p_0=2$ significantly accelerates the convergence compared to these two baseline methods. Please refer to appendix for further quantitative comparisons.

\begin{figure}[t]
    \centering
    \begin{tabular}{ccccc}
    \hspace{-0.1cm}\includegraphics[trim={0.2cm .2cm .2cm .2cm},clip,width=0.18\textwidth]{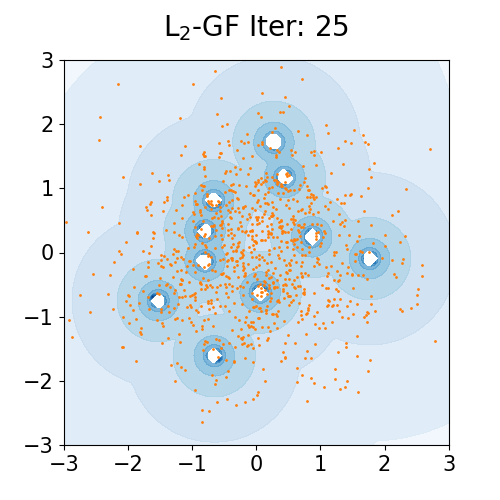}  
        & \hspace{-0.1cm}\includegraphics[trim={0.2cm .2cm .2cm .2cm},clip,width=0.18\textwidth]{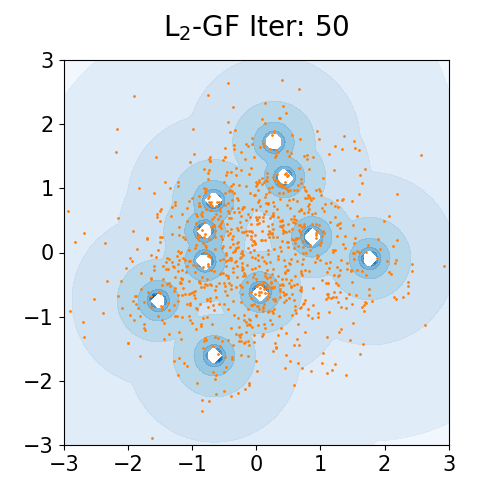}  
        & \hspace{-0.1cm}\includegraphics[trim={0.2cm .2cm .2cm .2cm},clip,width=0.18\textwidth]{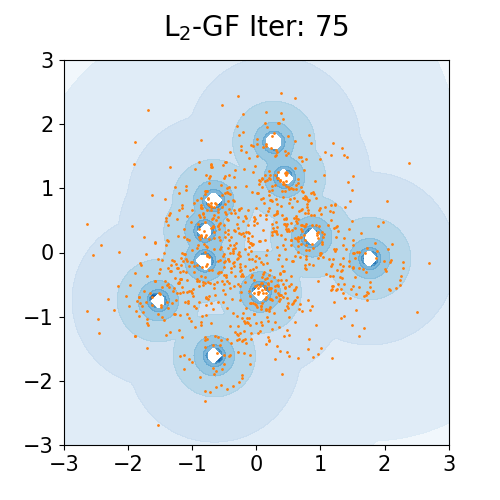}
        & \hspace{-0.1cm}\includegraphics[trim={0.2cm .2cm .2cm .2cm},clip,width=0.18\textwidth]{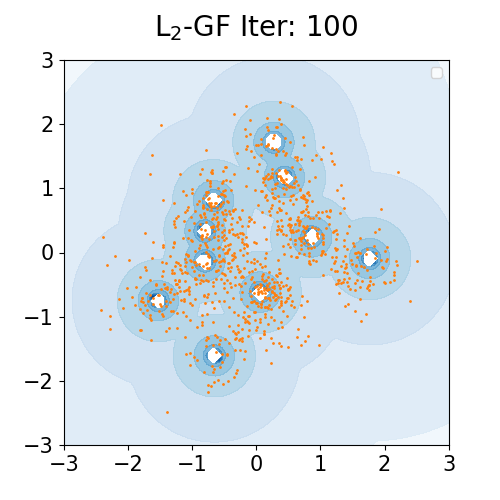}
        & \hspace{-0.1cm}\includegraphics[trim={0.2cm .2cm .2cm .2cm},clip,width=0.18\textwidth]{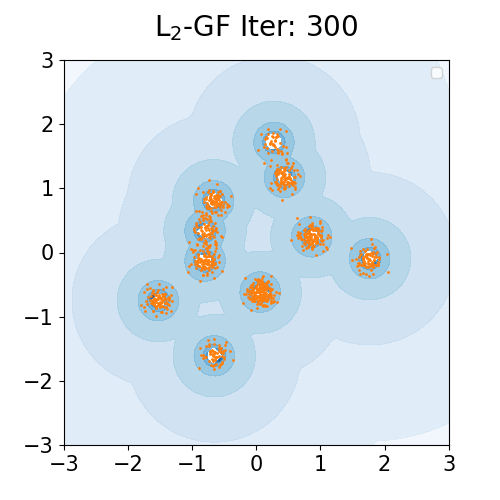}
\\
\hspace{-0.1cm}\includegraphics[trim={0.2cm .2cm .2cm .2cm},clip,width=0.18\textwidth]{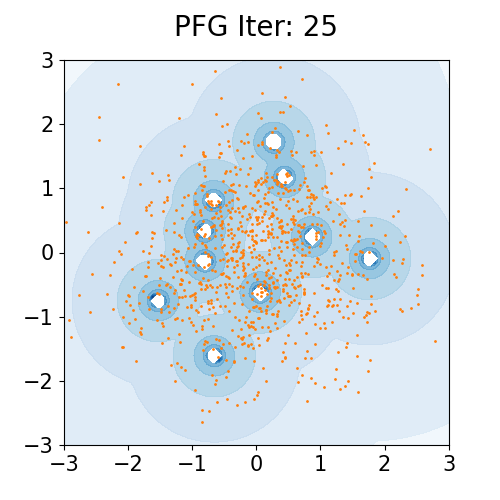}  
        &
  \hspace{-0.1cm}\includegraphics[trim={0.2cm .2cm .2cm .2cm},clip,width=0.18\textwidth]{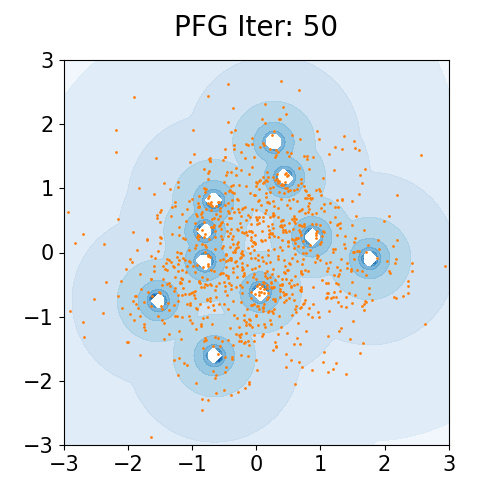}  
        & \hspace{-0.1cm}\includegraphics[trim={0.2cm .2cm .2cm .2cm},clip,width=0.18\textwidth]{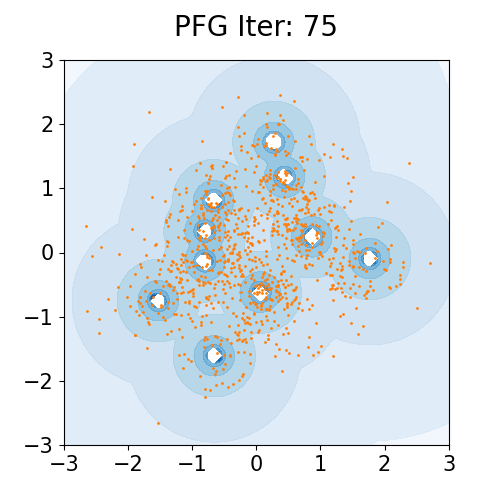}
        & \hspace{-0.1cm}\includegraphics[trim={0.2cm .2cm .2cm .2cm},clip,width=0.18\textwidth]{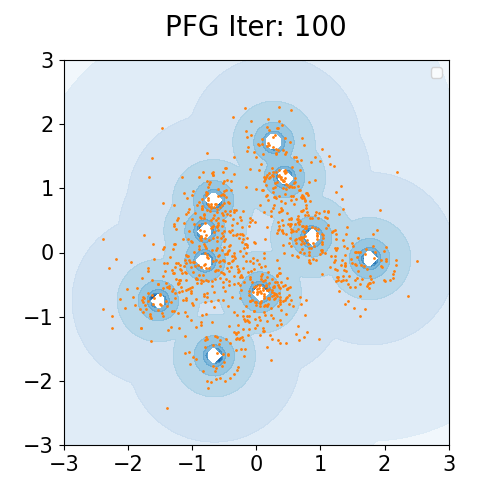}
        & \hspace{-0.1cm}\includegraphics[trim={0.2cm .2cm .2cm .2cm},clip,width=0.18\textwidth]{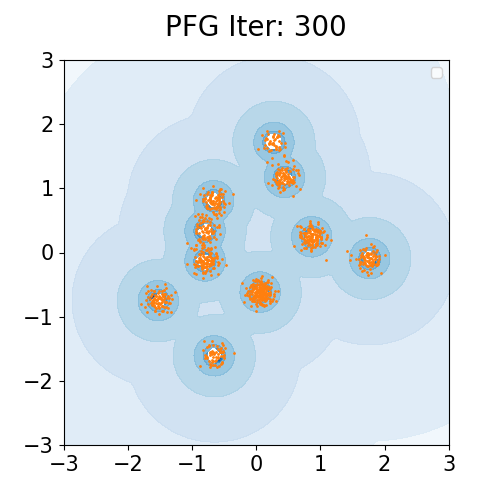}
        \\
\hspace{-0.1cm}\includegraphics[trim={0.2cm .2cm .2cm .2cm},clip,width=0.18\textwidth]{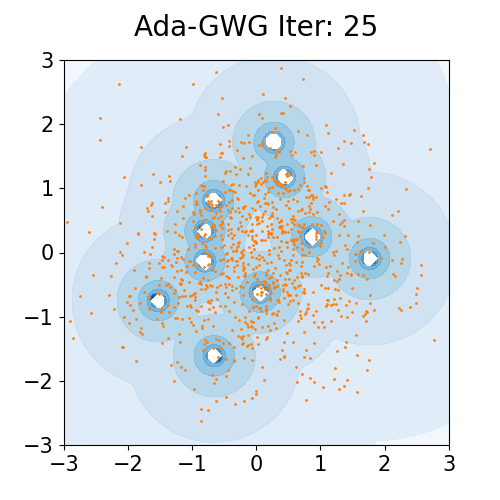}  
        &
  \hspace{-0.1cm}\includegraphics[trim={0.2cm .2cm .2cm .2cm},clip,width=0.18\textwidth]{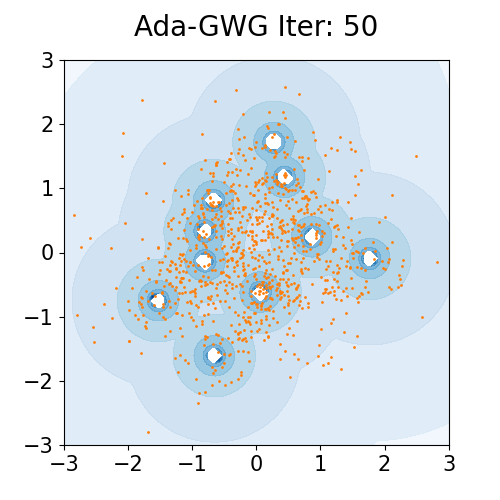}  
        & \hspace{-0.1cm}\includegraphics[trim={0.2cm .2cm .2cm .2cm},clip,width=0.18\textwidth]{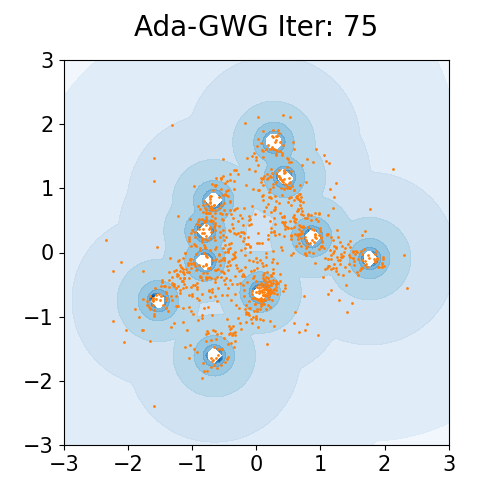}
        & \hspace{-0.1cm}\includegraphics[trim={0.2cm .2cm .2cm .2cm},clip,width=0.18\textwidth]{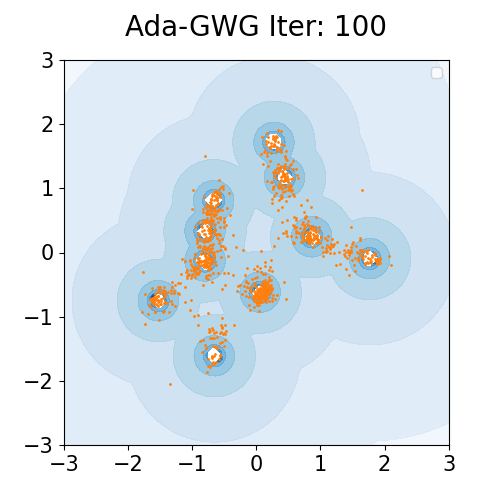}
        & \hspace{-0.1cm}\includegraphics[trim={0.2cm .2cm .2cm .2cm},clip,width=0.18\textwidth]{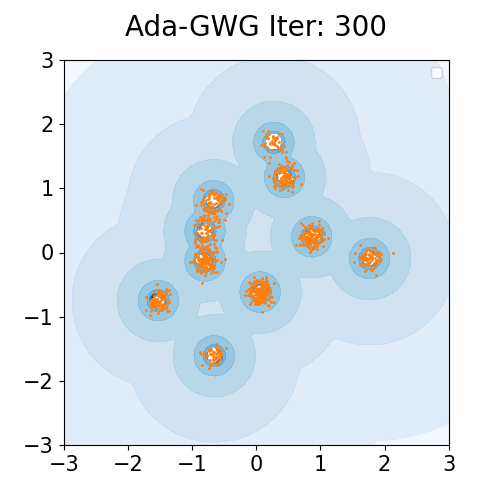}
        \\
    \end{tabular}
    \caption{Comparison of sampled particles at different numbers of iterations. {\bf Upper}: $L_2$-GF. {\bf Middle}: PFG. {\bf Lower}: Ada-GWG with $p_0=2$.
    }
    \label{fig:gmm}
    \vspace{-0.2cm}
\end{figure}

\subsection{Monomial Gamma}

To illustrate the effectiveness and efficiency of the adaptive method compared to the non-adaptive counterparts, we consider the heavy tailed Monomial Gamma distribution where the target $\pi \propto \exp(-0.3(|x_1|^{0.9}+|x_2|^{0.9}))$.

We test GWG and Ada-GWG with different choices of the initial values of $p$. The number of particles is $1000$. Figure \ref{fig:exp} demonstrates the KL divergence of different methods against the number of iterations. The dotted line represents GWG with fixed $p$, while the solid line represents the corresponding Ada-GWG that starts from the same $p$ at initialization.

We see that the adaptive method outperforms the non-adaptive counterpart consistently.
Moreover, Ada-GWG can automatically learn the appropriate value of $p$ especially when the initial values of $p$ is set inappropriately. 
For example, in our case, a relatively small value of $p=1.5$ would be inappropriate (the dotted green line) for GWG, while Ada-GWG with the same initial value of $p=1.5$ is able to provide much better approximation by automatically adjusting $p$ during runtime.
Consequently, Ada-GWG can exhibit greater robustness when determining the initial value of $p$.

\subsection{Conditioned Diffusion}

The conditioned diffusion example is a high-dimensional model arising from a Langevin SDE, with state $u:[0,1]\longrightarrow \reals$ and dynamics given by
\begin{equation}\label{eq:sde}
    du_t = \frac{10 u(1-u^2)}{1+u^2}dt + dx_t,\quad u_0=0,
\end{equation}
where $x=(x_t)_{t\geq 0}$ is a standard Brownian motion.

\begin{figure}[ht]
\vspace{-12mm}
    \centering
    \hspace{-3mm}
    \begin{minipage}[t]{0.52\textwidth}
        \centering
        \includegraphics[width=0.8\textwidth]{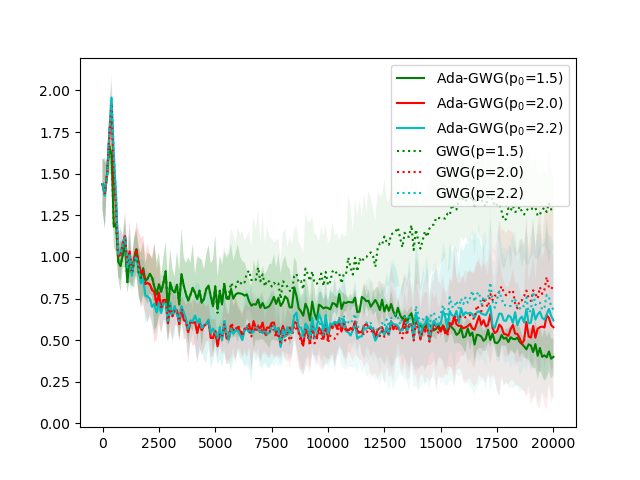} 
        \caption{KL divergence of different methods. Solid line: Ada-GWG. Dotted line: GWG counterpart.}
        \label{fig:exp}
    \end{minipage}
    \hspace{1mm} 
    \begin{minipage}[t]{0.47\textwidth}
        \centering
        \includegraphics[width=0.75\textwidth]{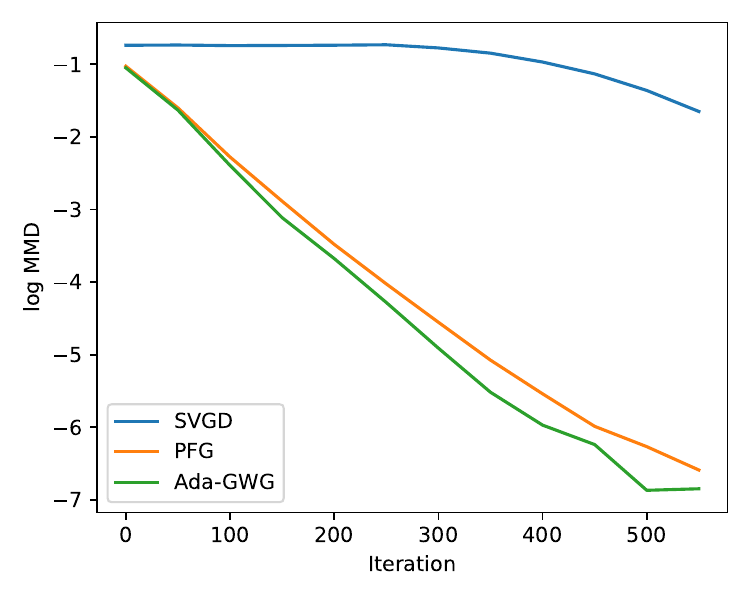}
        \caption{Comparison among PFG, Ada-GWG, SVGD in conditioned diffusion example.}
        \label{fig:cond_diff}
    \end{minipage}
    \hspace{-3mm}
    \vspace{-3mm}
\end{figure}

This system is commonly used in molecular dynamics to represent the motion of a particle with negligible mass trapped in an energy potential with thermal fluctuations represented by the Brownian forcing \cite{detommaso2018stein, cui2016dimension}. Given the perturbed observations $y$, the goal is to infer the posterior of the driving process $p(x|y)$. The forward operator is defined by $\mathcal{F}(x) = (u_{t_1}, \cdots, u_{t_{20}}) \in \reals^{20}$, where $t_i = 0.05 i$.
This is achieved by discretizing the above SDE \eqref{eq:sde} using an Euler-Maruyama scheme with step size $\Delta t = 0.01$; therefore the dimensionality of the problem is $100$.
The noisy observations are obtained as $y = \mathcal{F}(x_{\text{true}}) + \xi\in \reals^{20}$, where $x_{\text{true}}$ is a Brownian motion path and $\xi\sim\mathcal{N}(0,\sigma^2 I)$ with $\sigma=0.1$.
The prior is given by the Brownian motion $x = (x_t)_{t\geq 0}$.

We test three algorithms: PFG, Ada-GWG, and SVGD, with $n=1000$ particles.
To obtain the ground truth posterior, we run LMC with $1000$ particles in parallel, using a small step size $h=10^{-4}$ for $10000$ iterations. Figure \ref{fig:cond_diff} reports the logarithmic Maximum Mean Discrepancy (MMD) curves against iterations.
We observe that Ada-GWG provides best performance compared to the other methods.

\subsection{Bayesian Neural Networks}

We compare our algorithm with SGLD and SVGD variants on Bayesian neural networks (BNN).
Following \citet{Liu2016SVGD}, we conduct the two-layer network with 50 hidden units and ReLU activation function, and we use a $\mathrm{Gamma}(1,0.1)$ prior for the inverse covariances.
The datasets are all randomly partitioned into 90\% for training and 10\% for testing. The mini-batch size is 100 except for Concrete on which we use 400. The particle size is 100 and the results are averaged over 10 random trials. 
Table \ref{bnn} shows the average test RMSE and NLL and their standard deviation. 
We see that Ada-GWG can achieve comparable or better results than the other methods.
And the adaptive method consistently improves over $L_2$-GF. Figure \ref{fig:rmse-boston} shows the test RMSE against iterations of different methods on the Boston dataset. We can see that for this specific task, setting $p=3$ produces better results than when $p=2$. Although $L_2$-GF (\ie, GWG with $p=2$) is sub-optimal, our adaptive method (\ie, Ada-GWG with $p_0=2$) makes significant improvements and demonstrates comparable performance to the optimal choice of $p=3$. This suggests that our adaptive method is robust even if the initial exponent choice is not ideal. More comparisons of convergence results and hyperparameter tuning details can be found in the appendix.

\begin{table}[H]
  \captionof{table}{Averaged test RMSE and test negative log-likelihood of Bayesian Neural Networks on several UCI datasets. The results are averaged from 10 independent runs.}  
\label{bnn}
\begin{center}
\setlength\tabcolsep{3.3pt}
\begin{footnotesize}
\begin{sc}
\scalebox{0.85}{
\begin{tabular}{l|cccc|cccc}
\hline
& \multicolumn{4}{c|}{Avg. Test RMSE} & \multicolumn{4}{c}{Avg. Test NLL}\\
{\bf Dataset} & {\bf SGLD}& {\bf SVGD} & {\bf $L_2$-GF}& {\bf Ada-GWG} & {\bf SGLD}& {\bf SVGD} & {\bf $L_2$-GF}& {\bf Ada-GWG} \\
\hline
Boston       & $3.011_{\pm0.15}$ & $2.774_{\pm0.08}$ & $3.072_{\pm0.10}$& $\bm{2.721}_{\pm0.08}$    & $2.496_{\pm0.03}$ & $2.444_{\pm0.02}$ & $2.547_{\pm0.14}$ & $\bm{2.434}_{\pm0.02}$\\
Concrete     & $5.583_{\pm0.25}$ & $4.436_{\pm0.08}$ & $4.343_{\pm0.11}$& $\bm{3.871}_{\pm0.10}$    & $3.184_{\pm0.04}$ & $3.035_{\pm0.02}$ & $3.053_{\pm0.03}$ & $\bm{2.826}_{\pm0.02}$\\
Power        & $4.089_{\pm0.11}$ & $3.972_{\pm0.02}$ & $4.014_{\pm0.02}$& $\bm{3.944}_{\pm0.01}$    & $2.840_{\pm0.02}$ & $2.809_{\pm0.01}$ & $2.824_{\pm0.01}$ & $\bm{2.802}_{\pm0.01}$\\
Winewhite    & $0.677_{\pm0.01}$ & $0.664_{\pm0.01}$ & $0.666_{\pm0.01}$& $\bm{0.660}_{\pm0.01}$    & $1.033_{\pm0.01}$ & $1.014_{\pm0.01}$ & $1.015_{\pm0.01}$ & $\bm{1.006}_{\pm0.01}$\\
Winered      & $0.600_{\pm0.01}$ & $0.579_{\pm0.01}$ & $0.581_{\pm0.01}$& $\bm{0.575}_{\pm0.01}$    & $0.910_{\pm0.01}$ & $0.887_{\pm0.02}$ & $0.860_{\pm0.02}$ & $\bm{0.839}_{\pm0.02}$\\
protein      & $\bm{4.560}_{\pm0.04}$ & $4.779_{\pm0.03}$ & $4.867_{\pm0.01}$& $4.686_{\pm0.02}$    & $\bm{2.934}_{\pm0.01}$ & $2.984_{\pm0.01}$ & $3.003_{\pm0.00}$ & $2.964_{\pm0.00}$\\
\hline
\end{tabular}
}
\end{sc}
\end{footnotesize}
\end{center}

\end{table}

\begin{figure}[ht]
    \centering
    \includegraphics[scale=0.35]{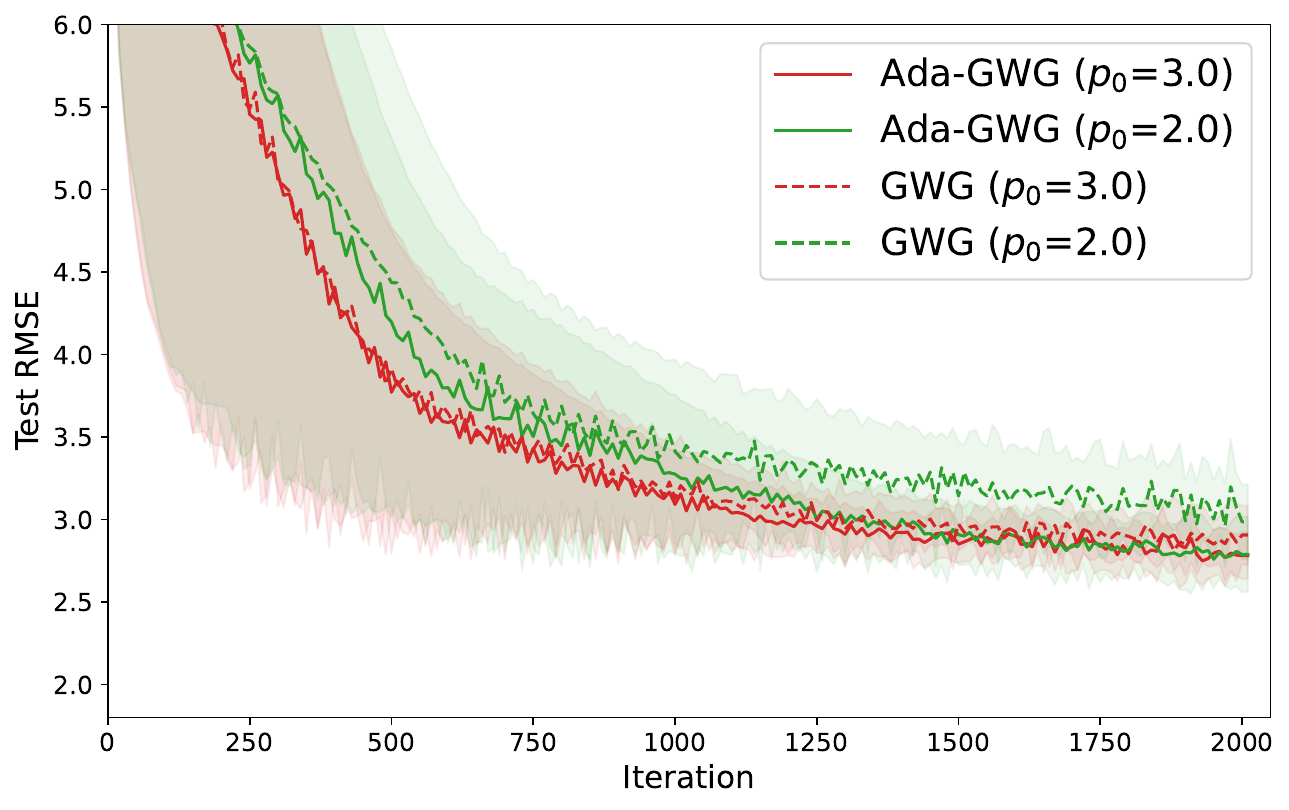}
    \caption{Test RMSE for the Bayesian Neural Networks on Boston dataset. The number in parentheses specifies the initial exponent $p_0$. The results are averaged from 10 independent runs.}
    \label{fig:rmse-boston}
\end{figure}

\section{Conclusion}
We introduced a new ParVI method, called GWG, which corresponds to a generalized Wasserstein gradient flow of KL divergence.
We show that our method has strong convergence guarantees in discrete time setting.
We also propose an adaptive version, called Ada-GWG, that can automatically adjust the Wassertein metric to accelerate convergence.
Extensive numerical results showed that Ada-GWG outperforms conventional ParVI methods.

\section*{Acknowledgements}
This work was supported by National Natural Science Foundation of China (grant no. 12201014 and grant no. 12292983).
The research of Cheng Zhang was supported in part by National Engineering Laboratory for Big Data Analysis and Applications, the Key Laboratory of Mathematics and Its Applications (LMAM) and the Key Laboratory of Mathematical Economics and Quantitative Finance (LMEQF) of Peking University. 
Ziheng Cheng and Shiyue Zhang are partially supported by the elite undergraduate training program of School of Mathematical Sciences in Peking University.
The authors are grateful for the computational resources provided by the High-performance Computing Platform of Peking University.
The authors appreciate the anonymous NeurIPS reviewers for their constructive feedback.

\bibliographystyle{nips_bib}
\bibliography{Bibliography}

\begin{thebibliography}{46}
\providecommand{\natexlab}[1]{#1}
\providecommand{\url}[1]{\texttt{#1}}
\expandafter\ifx\csname urlstyle\endcsname\relax
  \providecommand{\doi}[1]{doi: #1}\else
  \providecommand{\doi}{doi: \begingroup \urlstyle{rm}\Url}\fi

\bibitem[Adamczak et~al.(2017)Adamczak, Bednorz, and Wolff]{adamczak2017moment}
Rados{\l}aw Adamczak, Witold Bednorz, and Pawe{\l} Wolff.
\newblock Moment estimates implied by modified log-sobolev inequalities.
\newblock \emph{ESAIM: Probability and Statistics}, 21:\penalty0 467--494, 2017.

\bibitem[Alvarez-Melis et~al.(2022)Alvarez-Melis, Schiff, and Mroueh]{alvarez-melis2022optimizing}
David Alvarez-Melis, Yair Schiff, and Youssef Mroueh.
\newblock Optimizing functionals on the space of probabilities with input convex neural networks.
\newblock \emph{Transactions on Machine Learning Research}, 2022.
\newblock ISSN 2835-8856.
\newblock URL \url{https://openreview.net/forum?id=dpOYN7o8Jm}.

\bibitem[Ambrosio et~al.(2005)Ambrosio, Gigli, and Savar{\'e}]{ambrosio2005gradient}
Luigi Ambrosio, Nicola Gigli, and Giuseppe Savar{\'e}.
\newblock \emph{Gradient flows: in metric spaces and in the space of probability measures}.
\newblock Springer Science \& Business Media, 2005.

\bibitem[Balasubramanian et~al.(2022)Balasubramanian, Chewi, Erdogdu, Salim, and Zhang]{Balasubramanian2022TowardsAT}
Krishnakumar Balasubramanian, Sinho Chewi, Murat~A. Erdogdu, Adil Salim, and Matthew~Shunshi Zhang.
\newblock Towards a theory of non-log-concave sampling: First-order stationarity guarantees for langevin monte carlo.
\newblock In \emph{Annual Conference Computational Learning Theory}, 2022.

\bibitem[Barp et~al.(2019)Barp, Briol, Duncan, Girolami, and Mackey]{barp2019minimum}
Alessandro Barp, Francois-Xavier Briol, Andrew Duncan, Mark Girolami, and Lester Mackey.
\newblock Minimum stein discrepancy estimators.
\newblock \emph{Advances in Neural Information Processing Systems}, 32, 2019.

\bibitem[Barthe \& Roberto(2008)Barthe and Roberto]{barthe2008modified}
Franck Barthe and Cyril Roberto.
\newblock Modified logarithmic sobolev inequalities on.
\newblock \emph{Potential Analysis}, 29\penalty0 (2):\penalty0 167, 2008.

\bibitem[Blei et~al.(2016)Blei, Kucukelbir, and McAuliffe]{Blei2016VariationalIA}
David~M. Blei, Alp Kucukelbir, and Jon~D. McAuliffe.
\newblock Variational inference: A review for statisticians.
\newblock \emph{Journal of the American Statistical Association}, 112:\penalty0 859 -- 877, 2016.

\bibitem[Bobkov \& Ledoux(2000)Bobkov and Ledoux]{bobkov2000brunn}
Sergey~G Bobkov and Michel Ledoux.
\newblock From brunn-minkowski to brascamp-lieb and to logarithmic sobolev inequalities.
\newblock \emph{Geometric and Functional Analysis}, 10:\penalty0 1028--1052, 2000.

\bibitem[Chen et~al.(2018)Chen, Zhang, Wang, Li, and Chen]{Chen2018unified}
Changyou Chen, Ruiyi Zhang, Wenlin Wang, Bai Li, and Liqun Chen.
\newblock A unified particle-optimization framework for scalable bayesian sampling.
\newblock \emph{ArXiv}, abs/1805.11659, 2018.

\bibitem[Chen et~al.(2014)Chen, Fox, and Guestrin]{SGHMC}
Tianqi Chen, Emily Fox, and Carlos Guestrin.
\newblock Stochastic gradient hamiltonian monte carlo.
\newblock In \emph{International Conference on Machine Learning}, pp.\  1683--1691, 2014.

\bibitem[Chewi et~al.(2020)Chewi, Gouic, Lu, Maunu, and Rigollet]{Chewi2020chi-squared}
Sinho Chewi, Thibaut~Le Gouic, Chen Lu, Tyler Maunu, and Philippe Rigollet.
\newblock Svgd as a kernelized wasserstein gradient flow of the chi-squared divergence.
\newblock \emph{Advances in Neural Information Processing Systems}, 33, 2020.

\bibitem[Cui et~al.(2016)Cui, Law, and Marzouk]{cui2016dimension}
Tiangang Cui, Kody~JH Law, and Youssef~M Marzouk.
\newblock Dimension-independent likelihood-informed mcmc.
\newblock \emph{Journal of Computational Physics}, 304:\penalty0 109--137, 2016.

\bibitem[Dalalyan \& Karagulyan(2019)Dalalyan and Karagulyan]{dalalyan2019user}
Arnak~S Dalalyan and Avetik Karagulyan.
\newblock User-friendly guarantees for the langevin monte carlo with inaccurate gradient.
\newblock \emph{Stochastic Processes and their Applications}, 129\penalty0 (12):\penalty0 5278--5311, 2019.

\bibitem[Detommaso et~al.(2018)Detommaso, Cui, Marzouk, Spantini, and Scheichl]{detommaso2018stein}
Gianluca Detommaso, Tiangang Cui, Youssef Marzouk, Alessio Spantini, and Robert Scheichl.
\newblock A stein variational newton method.
\newblock \emph{Advances in Neural Information Processing Systems}, 31, 2018.

\bibitem[Devroye et~al.(2018)Devroye, Mehrabian, and Reddad]{devroye2018total}
Luc Devroye, Abbas Mehrabian, and Tommy Reddad.
\newblock The total variation distance between high-dimensional gaussians.
\newblock \emph{arXiv preprint arXiv:1810.08693}, 6, 2018.

\bibitem[di~Langosco et~al.(2021)di~Langosco, Fortuin, and Strathmann]{Langosco2021}
Lauro~Langosco di~Langosco, Vincent Fortuin, and Heiko Strathmann.
\newblock Neural variational gradient descent.
\newblock \emph{ArXiv}, abs/2107.10731, 2021.

\bibitem[Dong et~al.(2023)Dong, Wang, Yong, and Zhang]{dong2023particlebased}
Hanze Dong, Xi~Wang, LIN Yong, and Tong Zhang.
\newblock Particle-based variational inference with preconditioned functional gradient flow.
\newblock In \emph{The Eleventh International Conference on Learning Representations}, 2023.
\newblock URL \url{https://openreview.net/forum?id=6OphWWAE3cS}.

\bibitem[Duane et~al.(1987)Duane, Kennedy, Pendleton, and Roweth]{duane87}
S.~Duane, A.~D. Kennedy, B~J. Pendleton, and D.~Roweth.
\newblock {Hybrid Monte Carlo}.
\newblock \emph{Physics Letters B}, 195\penalty0 (2):\penalty0 216 -- 222, 1987.

\bibitem[Fan et~al.(2022)Fan, Zhang, Taghvaei, and Chen]{fan2022variational}
Jiaojiao Fan, Qinsheng Zhang, Amirhossein Taghvaei, and Yongxin Chen.
\newblock Variational wasserstein gradient flow.
\newblock In \emph{International Conference on Machine Learning}, pp.\  6185--6215. PMLR, 2022.

\bibitem[Gentil et~al.(2005)Gentil, Guillin, and Miclo]{gentil2005modified}
Ivan Gentil, Arnaud Guillin, and Laurent Miclo.
\newblock Modified logarithmic sobolev inequalities and transportation inequalities.
\newblock \emph{Probability theory and related fields}, 133:\penalty0 409--436, 2005.

\bibitem[Gentil et~al.(2007)Gentil, Guillin, and Miclo]{gentil2007modified}
Ivan Gentil, Arnaud Guillin, and Laurent Miclo.
\newblock Modified logarithmic sobolev inequalities in null curvature.
\newblock \emph{Revista Matematica Iberoamericana}, 23\penalty0 (1):\penalty0 235--258, 2007.

\bibitem[Grathwohl et~al.(2020)Grathwohl, Wang, Jacobsen, Duvenaud, and Zemel]{grathwohl2020}
Will Grathwohl, Kuan-Chieh Wang, J{\"o}rn-Henrik Jacobsen, David Duvenaud, and Richard Zemel.
\newblock Learning the stein discrepancy for training and evaluating energy-based models without sampling.
\newblock In \emph{International Conference on Machine Learning}, pp.\  3732--3747. PMLR, 2020.

\bibitem[Hu et~al.(2018)Hu, Chen, Sun, Bai, Ye, and Cheng]{Hu18}
Tianyang Hu, Zixiang Chen, Hanxi Sun, Jincheng Bai, Mao Ye, and Guang Cheng.
\newblock Stein neural sampler.
\newblock \emph{arXiv preprint arXiv:1810.03545}, 2018.

\bibitem[Hutchinson(1989)]{hutchinson1989stochastic}
Michael~F Hutchinson.
\newblock A stochastic estimator of the trace of the influence matrix for laplacian smoothing splines.
\newblock \emph{Communications in Statistics-Simulation and Computation}, 18\penalty0 (3):\penalty0 1059--1076, 1989.

\bibitem[Jordan et~al.(1999)Jordan, Ghahramani, Jaakkola, and Saul]{Jordan1999AnIT}
Michael~I. Jordan, Zoubin Ghahramani, T.~Jaakkola, and Lawrence~K. Saul.
\newblock An introduction to variational methods for graphical models.
\newblock \emph{Machine Learning}, 37:\penalty0 183--233, 1999.

\bibitem[Jordan et~al.(1998)Jordan, Kinderlehrer, and Otto]{jordan1998variational}
Richard Jordan, David Kinderlehrer, and Felix Otto.
\newblock The variational formulation of the fokker--planck equation.
\newblock \emph{SIAM journal on mathematical analysis}, 29\penalty0 (1):\penalty0 1--17, 1998.

\bibitem[Koehler et~al.(2023)Koehler, Heckett, and Risteski]{koehler2023statistical}
Frederic Koehler, Alexander Heckett, and Andrej Risteski.
\newblock Statistical efficiency of score matching: The view from isoperimetry.
\newblock In \emph{The Eleventh International Conference on Learning Representations}, 2023.
\newblock URL \url{https://openreview.net/forum?id=TD7AnQjNzR6}.

\bibitem[Korba et~al.(2020)Korba, Salim, Arbel, Luise, and Gretton]{korba2020non}
Anna Korba, Adil Salim, Michael Arbel, Giulia Luise, and Arthur Gretton.
\newblock A non-asymptotic analysis for stein variational gradient descent.
\newblock \emph{Advances in Neural Information Processing Systems}, 33:\penalty0 4672--4682, 2020.

\bibitem[Liu et~al.(2019)Liu, Zhuo, Cheng, Zhang, Zhu, and Carin]{Liu2019Understanding}
Chang Liu, Jingwei Zhuo, Pengyu Cheng, Ruiyi Zhang, Jun Zhu, and Lawrence Carin.
\newblock Understanding and accelerating particle-based variational inference.
\newblock \emph{International Conference on Machine Learning. PMLR}, pp. 4082-4092, 2019.

\bibitem[Liu(2017)]{Liu2017SVGF}
Qiang Liu.
\newblock Stein variational gradient descent as gradient flow.
\newblock \emph{Advances in neural information processing systems}, 30, 2017.

\bibitem[Liu \& Wang(2016)Liu and Wang]{Liu2016SVGD}
Qiang Liu and Dilin Wang.
\newblock Stein variational gradient descent: A general purpose bayesian inference algorithm.
\newblock \emph{Advances in neural information processing systems}, 29, 2016.

\bibitem[Mou et~al.(2022)Mou, Flammarion, Wainwright, and Bartlett]{mou2022improved}
Wenlong Mou, Nicolas Flammarion, Martin~J Wainwright, and Peter~L Bartlett.
\newblock Improved bounds for discretization of langevin diffusions: Near-optimal rates without convexity.
\newblock \emph{Bernoulli}, 28\penalty0 (3):\penalty0 1577--1601, 2022.

\bibitem[Mulholland(1949)]{mulholland1949generalizations}
HP~Mulholland.
\newblock On generalizations of minkowski's inequality in the form of a triangle inequality.
\newblock \emph{Proceedings of the London mathematical society}, 2\penalty0 (1):\penalty0 294--307, 1949.

\bibitem[Neal(2011)]{Neal2011-yo}
Radford Neal.
\newblock {MCMC} using hamiltonian dynamics.
\newblock In S~Brooks, A~Gelman, G~Jones, and XL~Meng (eds.), \emph{Handbook of Markov Chain Monte Carlo}, Chapman \& Hall/CRC Handbooks of Modern Statistical Methods. Taylor \& Francis, 2011.
\newblock ISBN 9781420079425.
\newblock URL \url{http://books.google.com/books?id=qfRsAIKZ4rIC}.

\bibitem[Newey \& McFadden(1986)Newey and McFadden]{Newey1986LargeSE}
Whitney Newey and Daniel McFadden.
\newblock Large sample estimation and hypothesis testing.
\newblock \emph{Handbook of Econometrics}, 4:\penalty0 2111--2245, 1986.

\bibitem[Santambrogio(2017)]{santambrogio2017euclidean}
Filippo Santambrogio.
\newblock $\{$Euclidean, metric, and Wasserstein$\}$ gradient flows: an overview.
\newblock \emph{Bulletin of Mathematical Sciences}, 7:\penalty0 87--154, 2017.

\bibitem[Song et~al.(2020)Song, Garg, Shi, and Ermon]{song2020sliced}
Yang Song, Sahaj Garg, Jiaxin Shi, and Stefano Ermon.
\newblock Sliced score matching: A scalable approach to density and score estimation.
\newblock In \emph{Uncertainty in Artificial Intelligence}, pp.\  574--584. PMLR, 2020.

\bibitem[Vempala \& Wibisono(2019)Vempala and Wibisono]{NEURIPS2019_65a99bb7}
Santosh Vempala and Andre Wibisono.
\newblock Rapid convergence of the unadjusted langevin algorithm: Isoperimetry suffices.
\newblock In H.~Wallach, H.~Larochelle, A.~Beygelzimer, F.~d\textquotesingle Alch\'{e}-Buc, E.~Fox, and R.~Garnett (eds.), \emph{Advances in Neural Information Processing Systems}, volume~32. Curran Associates, Inc., 2019.
\newblock URL \url{https://proceedings.neurips.cc/paper_files/paper/2019/file/65a99bb7a3115fdede20da98b08a370f-Paper.pdf}.

\bibitem[Villani(2021)]{villani2021topics}
C{\'e}dric Villani.
\newblock \emph{Topics in optimal transportation}, volume~58.
\newblock American Mathematical Soc., 2021.

\bibitem[Villani et~al.(2009)]{villani2009optimal}
C{\'e}dric Villani et~al.
\newblock \emph{Optimal transport: old and new}, volume 338.
\newblock Springer, 2009.

\bibitem[Wainwright \& Jordan(2008)Wainwright and Jordan]{Wainwright08}
M.~J. Wainwright and M.~I. Jordan.
\newblock Graphical models, exponential families, and variational inference.
\newblock \emph{Foundations and Trends in Maching Learning}, 1\penalty0 (1-2):\penalty0 1--305, 2008.

\bibitem[Wang et~al.(2018)Wang, Liu, and Liu]{wang2018variational}
Dilin Wang, Hao Liu, and Qiang Liu.
\newblock Variational inference with tail-adaptive f-divergence.
\newblock \emph{Advances in Neural Information Processing Systems}, 31, 2018.

\bibitem[Wang \& Li(2020)Wang and Li]{Wang2020InformationNF}
Yifei Wang and Wuchen Li.
\newblock Information newton's flow: second-order optimization method in probability space.
\newblock \emph{ArXiv}, abs/2001.04341, 2020.

\bibitem[Wang et~al.(2022)Wang, Chen, Pilanci, and Li]{wang2022optimal}
Yifei Wang, Peng Chen, Mert Pilanci, and Wuchen Li.
\newblock Optimal neural network approximation of wasserstein gradient direction via convex optimization.
\newblock \emph{arXiv preprint arXiv:2205.13098}, 2022.

\bibitem[Welling \& Teh(2011)Welling and Teh]{SGLD}
Max Welling and Yee~W Teh.
\newblock Bayesian learning via stochastic gradient langevin dynamics.
\newblock In \emph{International Conference on Machine Learning}, pp.\  681--688, 2011.

\bibitem[Wibisono \& Yang(2022)Wibisono and Yang]{Wibisono2022ConvergenceIK}
Andre Wibisono and Kay Yang.
\newblock Convergence in kl divergence of the inexact langevin algorithm with application to score-based generative models.
\newblock \emph{ArXiv}, abs/2211.01512, 2022.

\end{thebibliography}


\newpage
\appendix
\section{Minimizing Movement Scheme}\label{app_sec:mms}

\subsection{Geometric Interpretation}

In fact, under some mild conditions of $g$, the transportation cost $W_{c_h}(\cdot,\cdot)$ can induce a Wasserstein metric and thus $\mathcal{P}_{c_h}(\reals^d):=\{\mu\in \mathcal{P}(\reals^d): \expect_{\mu}[g(\frac{2x}{h})] < \infty \}$ is indeed a Wasserstein space.

\begin{prop}\label{prop:metric}
    Let $g(\cdot)=g_0(\|\cdot\|)$ where $g_0: \reals^+\cup{\{0\}}\rightarrow \reals^+\cup{\{0\}}$ satisfies $g_0(0)=0$ and $\|\cdot\|$ can be any norm in $\reals^d$. Then $g_0^{-1}(W_{c_h}(\cdot,\cdot))$ is a metric on $\mathcal{P}_{c_h}(\reals^d)$ if $g_0$ satisfies: (1) $g_0$ is continuous and strictly increasing; (2) $g_0$ is convex; (3) $\log g_0(x)$ is a convex function of $\log x$.
\end{prop}
\begin{proof}
    Suppose $\pi,\mu,\nu\in \mathcal{P}_{c_h}$. It is obvious that $g_0^{-1}(W_{c_h}(\mu,\nu))=0$ if and only if $\mu=\nu$. Besides, $g_0^{-1}(W_{c_h}(\cdot,\cdot))$ is symmetric. In the rest part of proof we aim to show that $g_0^{-1}(W_{c_h}(\mu,\pi))+g_0^{-1}(W_{c_h}(\nu,\pi))\geq g_0^{-1}(W_{c_h}(\mu,\nu))$. By Gluing lemma \cite{villani2009optimal}, we can construct random variables $X\sim \pi, Y\sim \mu, Z\sim \nu$ such that $(X,Y)$, $(X,Z)$ are the optimal coupling of $(\pi, \mu)$ and $(\pi,\nu)$ for transportation cost $W_{c_h}$, respectively. Then we have
    \begin{equation*}
        \begin{aligned}
            g_0^{-1}(W_{c_h}(\mu,\nu))
            &\leq g_0^{-1}\left(\expect g_0\left(\frac{\|Y-Z\|}{h}\right)\right) \\
            &\leq g_0^{-1}\left(\expect g_0\left(\frac{\|X-Y\|+\|X-Z\|}{h}\right)\right) \\
            &\leq g_0^{-1}\left(\expect g_0\left(\frac{\|X-Y\|}{h} \right)\right) + g_0^{-1}\left(\expect g_0\left(\frac{\|X-Z\|}{h}\right)\right) \\
            &= g_0^{-1}(W_{c_h}(\mu,\pi)) + g_0^{-1}(W_{c_h}(\pi,\nu)).
        \end{aligned}
    \end{equation*}
    Here the last inequality is due to generalized Minkowski's inequality \cite{mulholland1949generalizations}.
\end{proof}

The conditions in Proposition \ref{prop:metric} are mild and the most common choices of Young function $g$ satisfy them, and hence can induce a Wasserstein space the generalized Wasserstein gradient flow. Some typical examples of $g_0$ include $|x|^p, \exp(ax^2)-1, x\exp(ax^b)$, while the norm $\|\cdot\|$ in $\reals^d$ can be $\|\cdot\|_p, \|\cdot\|_H$ and so on.

\subsection{Derivation of Generalized Wasserstein Gradient Flow}

\begin{thm}[Restatement of Theorem \ref{thm:wgf}]\label{thm:full_wgf}
    Given a continuously differentiable Young function $g$ and step size $h>0$, define cost function $c_h(x,y)=g(\frac{x-y}{h})h$. Suppose that $\pi, \mu_{kh}\in\mathcal{P}_{c_h}(\reals^d):=\{\mu\in \mathcal{P}(\reals^d): \expect_{\mu}[g(\frac{2x}{h})] < \infty \}$.
    If $g$ satisfies assumptions in Proposition \ref{prop:metric}, $\mathcal{P}_{c_h}(\reals^d)$ is a Wasserstein space equipped with Wasserstein metric. Consider MMS under transportation cost $W_{c_h}$:
    \begin{equation}\label{eq:wgf_mms}
        \mu_{(k+1)h} := \argmin_{\mu\in \mathcal{P}_{c_h}(\reals^d)} \KL(\mu\|\pi) + W_{c_h}(\mu, \mu_{kh}).
    \end{equation}
    Denote the optimal transportation map under $W_{c_h}$ from $\mu_{(k+1)h}$ to $\mu_{kh}$ by $T_k(\cdot)$. Then we have
    \begin{equation}
        \frac{T_k(x) - x}{h} = -\nabla g^*\left(\nabla\log\pi(x)-\nabla\log\mu_{(k+1)h}(x)\right).
    \end{equation}
\end{thm}
\begin{proof}
    By Kantorovich duality \citep{villani2021topics}, the optimal transportation cost \eqref{eq:wc_distance} has an equivalent definition:
    \begin{equation}\label{eq:ka_dual}
        W_{c} (\mu,\nu) = \sup_{\varphi} \int \varphi d\mu + \int \varphi^c d\nu,\ \text{where}\ \varphi^c(y):=\inf_{x\in \mathbb{R}^d} c(x,y)-\varphi(x).
    \end{equation}
    Take the functional derivative of the optimization problem \eqref{eq:wgf_mms} and define the optimal $\varphi$ in \eqref{eq:ka_dual} as $\psi$. The following holds:
    \begin{equation}\label{eq:func_dr}
        \frac{\delta}{\delta\mu} \KL(\mu_{(k+1)h}\|\pi) + \psi_{kh} = \text{const}.
    \end{equation}
    Here the reason for have a constant instead of zero is that we constrain $\mu_{(k+1)h}$ in the space of smooth probability density. Note that $c_h(x,y)\leq \frac{1}{2}\left(g(\frac{2x}{h})+g(\frac{2y}{h})\right)$ and $\mu_{kh}, \mu_{(k+1)h}\in \mathcal{P}_{c_h}(\reals^d)$, then by the fundamental theorem of optimal transportation \citep{villani2021topics},
    \begin{equation}
        \psi_{kh}(x)+\psi_{kh}^{c_h}(y)=c_h(x,y),\ \text{for}\ y=T_k(x),
    \end{equation}
    which implies $\nabla \psi_{kh}(x)=\nabla_xc_h(x,y)=\nabla g(\frac{x-y}{h})$, \ie, 
    \begin{equation}
        x-y=\nabla g^*(\nabla \psi_{kh}(x)) h.
    \end{equation}
    Combine this equation with \eqref{eq:func_dr} and thus the optimal map is given by 
    \begin{equation*}
        \begin{aligned}
            \frac{T_k(x)-x}{h}
            &= -\nabla g^*( -\nabla\frac{\delta}{\delta\mu} \KL(\mu_{(k+1)h}\|\pi)) \\
            &= -\nabla g^*(\nabla \log\frac{\pi}{\mu_{(k+1)h}}).
        \end{aligned}
    \end{equation*}
\end{proof}

\section{Details of Motivating Example}\label{app_sec:dme}
We use Example \ref{eg:Lq_mog} to illustrate the benefits of choosing general Young function $g^*$, which is also discussed in \cite{Balasubramanian2022TowardsAT,Wibisono2022ConvergenceIK}. 

We follow the procedures of \cite{Wibisono2022ConvergenceIK}. For convenience, let $\pi_0=\mathcal{N}(-m, 1),\pi_1=\mathcal{N}(m, 1)$ and rewrite $\pi=\frac{1}{2}\pi_0+\frac{1}{2}\pi_1, \mu=\frac{3}{4}\pi_0+\frac{1}{4}\pi_1$. The lower bound of KL divergence follows from \cite{devroye2018total} and \textit{Pinsker inequality}. In addition, \cite{Balasubramanian2022TowardsAT} shows 
\begin{equation}
    \nabla\log\pi - \nabla\log\mu = -m\frac{\pi_0\pi_1}{2\pi\mu}.
\end{equation}
Also note that $\frac{\pi_0}{\pi_1}=\exp(-2mx)$. Therefore for any $q\geq 1$, the following bound holds:
\begin{equation}\label{eq:sd_lb}
    \begin{aligned}
        \expect_{\mu}\bv\bv \nabla\log\frac{\pi}{\mu} \bv\bv_q^q
        &= \frac{m^q}{2^q}\int \frac{\pi_0^q\pi_1^q}{\mu^{q-1}\pi^q}dx \\
        &= 4^{q-1}m^q \int \frac{\pi_0^q\pi_1^q}{(3\pi_0+\pi_1)^{q-1}(\pi_0+\pi_1)^q}dx \\
        &\geq 4^{q-1}m^q \left( \int_{x\geq \frac{1}{m}} \frac{\pi_0^q}{(1+e^{-2})^q\pi_1^{q-1}(1+3e^{-2})^{q-1}}dx+ \int_{x\leq -\frac{1}{m}} \frac{\pi_1^q}{(1+e^{-2})^q\pi_0^{q-1}(3+e^{-2})^{q-1}}dx\right).
    \end{aligned}
\end{equation}

\begin{equation}\label{eq:pi_0/pi_1_lb}
    \begin{aligned}
        \int_{x\geq \frac{1}{m}} \frac{\pi_0^q}{\pi_1^{q-1}}dx
        &= \int_{x\geq \frac{1}{m}} \frac{1}{\sqrt{2\pi}}\exp(-\frac{1}{2}(x+m)^2-2m(q-1)x)dx \\
        &= \prob_{\mathcal{N}(0,1)}\{Z\geq (2q-1)m+\frac{1}{m}\} \exp(2q(q-1)m^2) \\
        &\geq \frac{3}{4}\frac{1}{(2q-1)m+\frac{1}{m}}\frac{1}{\sqrt{2\pi}}\exp\left(-\frac{1}{2}((2q-1)m+\frac{1}{m})^2+2q(q-1)m^2\right) \\
        &\geq \frac{3}{4}\frac{1}{2qm}\frac{1}{\sqrt{2\pi}}\exp(-\frac{1}{2}m^2-2q+\frac{1}{2}).
    \end{aligned}
\end{equation}
Here the first inequality is by $\prob_{\mathcal{N}(0,1)}\{Z\geq t\}\geq (\frac{1}{t}-\frac{1}{t^3})\frac{1}{\sqrt{2\pi}}\exp(-\frac{t^2}{2})$ for any $t>0$ and $(2q-1)m+\frac{1}{m}\geq 2$.
Similarly, we can prove that
\begin{equation}\label{eq:pi_1/pi_0_lb}
    \int_{x\leq -\frac{1}{m}} \frac{\pi_1^q}{\pi_0^{q-1}}dx \geq \frac{3}{4}\frac{1}{2qm}\frac{1}{\sqrt{2\pi}}\exp(-\frac{1}{2}m^2-2q+\frac{1}{2}).
\end{equation}
Plug \eqref{eq:pi_0/pi_1_lb} and \eqref{eq:pi_1/pi_0_lb} in \eqref{eq:sd_lb},
\begin{equation*}
    \begin{aligned}
        \expect_{\mu}\bv\bv \nabla\log\frac{\pi}{\mu} \bv\bv_q^q 
        &\geq \frac{1}{q}m^{q-1}\exp(-\frac{m^2}{2})\cdot \frac{3}{4^3}\sqrt{\frac{2e}{\pi}}4^{q}(e^2+1)^{-q}\left[(\frac{1}{1+3e^{-2}})^{q-1}+(\frac{1}{3+e^{-2}})^{q-1}\right] \\
        &\geq \frac{1}{q}m^{q-1}\exp(-\frac{m^2}{2})\cdot \frac{3}{4^3}\sqrt{\frac{2e}{\pi}}(1+3e^{-2})\left[\frac{4}{(1+e^2)(1+3e^{-2})}\right]^q \\
        &\geq \frac{0.08}{qm}(\frac{m}{3})^q\exp(-\frac{m^2}{2}).
    \end{aligned}
\end{equation*}
As for the upper bound,
\begin{equation}\label{eq:sd_ub}
    \begin{aligned}
        \expect_{\mu}\bv\bv \nabla\log\frac{\pi}{\mu} \bv\bv_q^q
        &= \frac{m^q}{2^q}\int \frac{\pi_0^q\pi_1^q}{\mu^{q-1}\pi^q}dx \\
        &= 4^{q-1}m^q \int \frac{\pi_0^q\pi_1^q}{(3\pi_0+\pi_1)^{q-1}(\pi_0+\pi_1)^q}dx \\
        &\leq 4^{q-1}m^q \left( \int_{x\geq 0} \frac{\pi_0^q}{\pi_1^{q-1}}dx+ \int_{x\leq 0} \frac{\pi_1^q}{\pi_0^{q-1}}dx\right).
    \end{aligned}
\end{equation}

\begin{equation}\label{eq:pi_0/pi_1_ub}
    \begin{aligned}
        \int_{x\geq 0} \frac{\pi_0^q}{\pi_1^{q-1}}dx
        &= \int_{x\geq 0} \frac{1}{\sqrt{2\pi}}\exp(-\frac{1}{2}(x+m)^2-2m(q-1)x)dx \\
        &= \prob_{\mathcal{N}(0,1)}\{Z\geq (2q-1)m\} \exp(2q(q-1)m^2) \\
        &\leq \frac{1}{(2q-1)m}\frac{1}{\sqrt{2\pi}}\exp\left(-\frac{1}{2}((2q-1)m)^2+2q(q-1)m^2\right) \\
        &\leq \frac{1}{qm}\frac{1}{\sqrt{2\pi}}\exp(-\frac{1}{2}m^2).
    \end{aligned}
\end{equation}
Here the first inequality is by $\prob_{\mathcal{N}(0,1)}\{Z\geq t\}\leq \frac{1}{t}\frac{1}{\sqrt{2\pi}}\exp(-\frac{t^2}{2})$ for any $t>0$. Similarly, we can prove that
\begin{equation}\label{eq:pi_1/pi_0_ub}
    \int_{x\leq 0} \frac{\pi_1^q}{\pi_0^{q-1}}dx \leq \frac{1}{qm}\frac{1}{\sqrt{2\pi}}\exp(-\frac{1}{2}m^2).
\end{equation}
Therefore
\begin{equation*}
    \begin{aligned}
        \expect_{\mu}\bv\bv \nabla\log\frac{\pi}{\mu} \bv\bv_q^q 
        &\leq 4^{q-1}m^q\cdot \frac{2}{qm}\frac{1}{\sqrt{2\pi}}\exp(-\frac{1}{2}m^2) \\
        &\leq \frac{0.2}{qm}(4m)^q\exp(-\frac{m^2}{2}).
    \end{aligned}
\end{equation*}

\section{Asymptotic Normality of Estimator}
In practice, there are finite particles and we can only get a Monte Carlo estimation of \eqref{eq:gsm_pop}. But our theoretical analysis is based on the population loss. With this concern, we show that the maximum point of estimation \eqref{eq:gsm} have good statistical properties. To be specific, the estimator converges to true maximum point with asymptotic normality under mild conditions. Similar properties are also studied in \cite{barp2019minimum, song2020sliced, koehler2023statistical}.

Define objective function $\ell(w, x):=\nabla \log{\pi(x)}^T f_w(x) + \nabla_x \cdot f_w(x) - g(f_w(x))$.
\begin{asp}\label{asp:unique}
    $\mathcal{W}$ is compact and $\mathcal{L}(\cdot)$ defined in \eqref{eq:gsm_pop} has a unique maximum point $w^*\in \text{int}(\mathcal{W})$.
\end{asp}
\begin{asp}\label{asp:continuous}
    $f_w(x), \nabla_x f_w(x)$ are continuous with $w$. $\max\left\{\|f_w(x)\|, |\nabla_x\cdot f_w(x)|, g(f_w(x))\right\}\leq M_0(x)$ for some $M_0\in \mathcal{L}^1(\mu)$.
\end{asp}
\begin{asp}\label{asp:twice_diff}
    There exists a neighborhood $\mathcal{N}$ of $w^*$ such that $\ell(\cdot,x)$ is twice differentiable in $\mathcal{N}$ and $\|\nabla^2 \ell(w, x)\|\leq M_1(x)$ for all $w\in \mathcal{N}$. Additionally, assume $M_1\in \mathcal{L}^1(\mu)$ and $H:=\nabla^2 \mathcal{L}(w^*)$ is non-singular. 
\end{asp}

\begin{thm}\label{thm:asp_normal}
    Given $x_1,\cdots, x_n\overset{\iid}{\sim}\mu$, let $\widehat{w}_n:=\argmin_{w} \widehat{\mathcal{L}}_n(w):=\frac{1}{n}\sum_{i=1}^{n} \ell(w,x_i)$. Under Assumption \ref{asp:unique}-\ref{asp:twice_diff}, we have
    \begin{equation}
        \sqrt{n}(\widehat{w}_n-w^*) \overset{p}{\to}  \mathcal{N}\left(0, H^{-1}\Sigma H^{-1}\right),
    \end{equation}
    where $\Sigma = \expect_{\mu} \nabla_w\ell(w^*,\cdot) \otimes \nabla_w\ell(w^*,\cdot)$.
\end{thm}
\begin{proof}
    Note that under Assumption \ref{asp:continuous}, $|\ell(w,x)|\leq M(x)$ for some $M\in \mathcal{L}^1(\mu)$. By \citet[Lemma 2.4, Theorem 2.1]{Newey1986LargeSE}, $\widehat{w}_n$ is weakly consistent for $w^*$. Additionally, since $w^*$ is the maximum point of $\mathcal{L}$ and by Central Limit Theorem,
    \begin{equation*}
        \sqrt{n}\nabla \widehat{\mathcal{L}}_n(w^*)=\frac{1}{\sqrt{n}}\sum_{i=1}^{n}\nabla_w \ell(w^*, x_i)\overset{p}{\to} \mathcal{N}(0, \Sigma).
    \end{equation*}
    By Assumption \ref{asp:twice_diff} and \citet[Lemma 2.4]{Newey1986LargeSE}, the second order derivative converges uniformly, \ie
    \begin{equation*}
        \sup_{w\in \mathcal{N}} \|\nabla^2\widehat{\mathcal{L}}_n(w)-\nabla^2\mathcal{L}(w)\|\overset{p}{\to} 0.
    \end{equation*}
    Finally, the result follows \citet[Theorem 3.1]{Newey1986LargeSE}.
    
\end{proof}

\section{Proof in Section \ref{sec:convergence_parvi}}\label{sec:pf_main}
We first justify our Assumption $\ref{asp:error_lp}$ and \ref{asp:error_l2}, then present some crucial lemmas and finish the proof of convergence results. Our proof procedure uses interpolation process of discrete dynamics, following \cite{NEURIPS2019_65a99bb7, Balasubramanian2022TowardsAT}. Informally, the difference between discrete dynamics and continuous dynamics consists of two parts: discretization error and estimation error (by neural nets). We bound the discretization error in Lemma \ref{lem:discrete_err_lp}, \ref{lem:discrete_err_l2} and the estimation error in Lemma \ref{lem:nn_err_lp}, \ref{lem:nn_err_l2}.

\subsection{Justification for Assumption \ref{asp:error_lp} and \ref{asp:error_l2}}\label{sec:asp_12}

\begin{proof}[Proof of Proposition \ref{prop:a1a2}]
    It suffices to show that there exists $\varepsilon_2<\infty$ such that for any $a,b\in \reals$, the following inequality holds:
    \begin{equation}
        |b-\sgn(a)|a|^{q-1}|^p\leq \varepsilon_1 \frac{|a|^q}{q} + \varepsilon_2\left(\frac{|a|^q}{q} -ab + \frac{|b|^p}{p}\right).
    \end{equation}
    If $a=0$, then $\varepsilon_2\geq p$ is sufficient. Without loss of generality, suppose $a=1$ (by replacing $b$ with $\sgn(a)b/|a|^{q-1}$). We only need to show that
    \begin{equation}\label{eq:goal}
        |b-1|^p\leq \varepsilon_1 + \varepsilon_2\left(\frac{1}{q} -b + \frac{|b|^p}{p}\right).
    \end{equation}
    
    (1) \textbf{Case} $p\geq 2$.

    Let $\varepsilon_1=0$. Since $\lim_{b\to 1} \frac{|b-1|^p}{\frac{1}{q} -b + \frac{|b|^p}{p}} = \lim_{b\to 1}\frac{2|b-1|^{p-2}}{p-1} \leq 2$, so there exists $\delta>0$ such that when $b\in [1-\delta,1+\delta]$, \eqref{eq:goal} holds if $\varepsilon_2\geq p+1>2$. Also note that $f(b)=\frac{|b-1|^p}{\frac{1}{q} -b + \frac{|b|^p}{p}}$ is a continuous function on $\reals\backslash (1-\delta,1+\delta) $ and $\lim_{b\to \infty} f(b)=p<+\infty$. Therefore, $f(b)$ is bounded on $\reals\backslash [1-\delta,1+\delta]$ and thus \eqref{eq:goal} holds for finite $\varepsilon_2$. It's obvious that $\varepsilon_2$ only depends on $p$ in this case.

    (2) \textbf{Case} $p<2$.
    
    Similarly, let $\delta=\varepsilon_1^{1/p}$. When $b\in [1-\delta, 1+\delta]$, \eqref{eq:goal} will trivially hold for any $\varepsilon_2>0$. Also, $f(b)$ is bounded on $\reals\backslash [1-\delta,1+\delta]$ and thus there exists  finite $\varepsilon_2$ determined by $p, \delta$ such that \eqref{eq:goal} holds.
\end{proof}

\subsection{Main Lemmas}

\begin{lemma}\label{lem:dKL}
    For any $t\in (kh, (k+1)h)$, $\partial_t \KL(\mu_t\|\pi) = - \expect_{\mu_t} \left\langle \nabla\log{\frac{\pi}{\mu_t}}, \expect[v_k(X_{kh})|X_t=\cdot] \right\rangle$
\end{lemma}

\begin{proof}
    Let $\mu_{t|\mathcal{F}_{kh}}$ denote the law of $X_t$ conditioned on the filtration $\mathcal{F}_{kh}$ at time $kh$. Then by Fokker-Planck equation, we have 
    \begin{equation*}
        \partial_t \mu_{t|\mathcal{F}_{kh}} = -\divg \left( \mu_{t|\mathcal{F}_{kh}} v_k(X_{kh}) \right).
    \end{equation*}
    Then we take expectation of the above equation; by Bayesian formula \citep{NEURIPS2019_65a99bb7},
    \begin{equation*}
        \begin{aligned}
            \partial_t \mu_t 
            &= -\divg\ \expect[\mu_{t|\mathcal{F}_{kh}} v_k(X_{kh})] \\
            &= -\divg \left( \mu_t \expect[v_k(X_{kh})|X_t=\cdot] \right)
        \end{aligned}
    \end{equation*}
    Hence 
    \begin{equation*}
        \partial_t \KL(\mu_t\|\pi) = -\expect_{\mu_t} \left\langle \nabla\log{\frac{\pi}{\mu_t}}, \expect[v_k(X_{kh})|X_t=\cdot] \right\rangle.
    \end{equation*}
\end{proof}

\begin{lemma}\label{lem:err_mu_t}
    Suppose that $h<\frac{1}{G_2}$. Under Assumption \ref{asp:smooth}, for any $t\in [kh, (k+1)h],\ q>1, \frac{1}{p} + \frac{1}{q}=1$, 
    \begin{equation*}
        \expect_{\mu_t}\bv\bv \nabla\log{\frac{\mu_t}{\mu_{kh}}} \bv\bv_q^q \leq \left(M_p(1-hG_2)^{-1}dh\right)^q
    \end{equation*}
\end{lemma}

\begin{proof}
    Note that $\log\frac{\mu_t}{\mu_{kh}}=-\log\det\left(I_d+(t-kh)\nabla v_k\right)$ since $x\mapsto x+(t-kh)v_k(x)$ is an orientation-preserving diffeomorphism under $h<\frac{1}{G_2}$. Then the following holds:
    \begin{equation*}
        \begin{aligned}
            \bv\bv \nabla\log\frac{\mu_t}{\mu_{kh}}(x)\bv\bv_q
            &= \sup_{\|z\|_p=1} \left\langle \nabla\log\frac{\mu_t}{\mu_{kh}}(x), z \right\rangle \\
            &= \sup_{\|z\|_p=1} \lim_{\delta\to 0} \bv \log\det\left( I_d+(t-kh)\nabla v_k(x+\delta z)\right) - \log\det\left( I_d+(t-kh)\nabla v_k(x)\right)\bv \\
            &\leq  \sup_{\|z\|_p=1} \lim_{\delta\to 0} (t-kh) \| \nabla v_k(x+\delta z) - \nabla v_k(x)\|_2(1-hG_2)^{-1}d \\
            &\leq M_p(1-hG_2)^{-1}dh.
        \end{aligned}
    \end{equation*}
    Here the first equation is due to Young's inequality and the second equation follows the definition of gradient. The inequality in the third line is due to Lemma \ref{lem:logdet} and the last one is by Assumption \ref{asp:smooth}. With this uniform bound we finish the proof.
\end{proof}

\begin{lemma}\label{lem:discrete_err_lp}
    Under Assumption \ref{asp:error_lp}, \ref{asp:smooth}, and the same conditions in Lemma \ref{lem:err_mu_t},
    \begin{equation*}
         \expect_{\mu_t} \bv\bv \nabla g^*(\nabla\log\frac{\pi}{\mu_t}) - \nabla g^*(\nabla\log\frac{\pi}{\mu_{kh}})\bv\bv_p^p \leq c_1 \expect_{\mu_t} \bv\bv \nabla\log\frac{\pi}{\mu_t}\bv\bv_q^q + c_2 \left( M_p(1-hG_2)^{-1}dh\right)^q
    \end{equation*}
    where $c_1,c_2$ are defined as:
    \begin{equation}\label{eq:def_c1c2}
        (c_1, c_2) = \left\{
        \begin{array}{ll}
             (0,\ 2^{p-q}) & \text{if}\ q\leq 2, \\
             \left(3^{-p},\ \min\left\{3^{q-p}(\frac{p}{q})^{\frac{q-p}{q-1}}(1-\frac{p}{q})^{\frac{q-p}{p}}, (q-1)^p\left(\frac{(\frac{4}{3})^{\frac{1}{q-1}}}{(\frac{4}{3})^{\frac{1}{q-1}}-1}\right)^{q-p} \right\}\right) & \text{otherwise}.
        \end{array} \right.
    \end{equation}
\end{lemma}

\begin{proof}
    Since $g^*(x) = \frac{1}{q}\|x\|_q^q$,  we have $\nabla g^*(x) = \sgn(x) \odot |x|^{q-1}$. Here $\odot$ means entry-wise product. Apply Lemma \ref{lem:c1c2} entry-wise and thus
    \begin{equation*}
        \begin{aligned}
            \expect_{\mu_t} \bv\bv \nabla g^*(\nabla\log\frac{\pi}{\mu_t}) - \nabla g^*(\nabla\log\frac{\pi}{\mu_{kh}})\bv\bv_p^p 
            &\leq c_1 \expect_{\mu_t} \bv\bv \nabla\log\frac{\pi}{\mu_t}\bv\bv_q^q + c_2 \expect_{\mu_t} \bv\bv \nabla\log\frac{\mu_t}{\mu_{kh}}\bv\bv_q^q \\
            &\leq c_1 \expect_{\mu_t} \bv\bv \nabla\log\frac{\pi}{\mu_t}\bv\bv_q^q + c_2 \left(M_p(1-hG_2)^{-1}dh\right)^q.
        \end{aligned}
     \end{equation*}
     The second inequality is due to Lemma \ref{lem:err_mu_t}.
\end{proof}

\begin{lemma}\label{lem:discrete_err_l2}
    Under Assumption \ref{asp:error_l2}, \ref{asp:smooth} and the same conditions in Lemma \ref{lem:err_mu_t}, 
     \begin{equation*}
         \expect_{\mu_t} \bv\bv \nabla g^*(\nabla\log\frac{\pi}{\mu_t}) - \nabla g^*(\nabla\log\frac{\pi}{\mu_{kh}})\bv\bv_2^2 \leq \beta^2 \left(M_2(1-hG_2)^{-1}dh\right)^2
     \end{equation*}
\end{lemma}

\begin{proof}
    Note that $g^*$ is $\beta$-smooth and by Lemma \ref{lem:err_mu_t},
    \begin{equation*}
        \begin{aligned}
            \expect_{\mu_t} \bv\bv \nabla g^*(\nabla\log\frac{\pi}{\mu_t}) - \nabla g^*(\nabla\log\frac{\pi}{\mu_{kh}})\bv\bv_2^2
            &\leq \beta^2\expect_{\mu_t} \bv\bv \nabla\log\frac{\mu_t}{\mu_{kh}}\bv\bv_2^2 \\
            &\leq \beta^2\left(M_2(1-hG_2)^{-1}dh\right)^2.
        \end{aligned}
    \end{equation*}
\end{proof}

\begin{lemma}\label{lem:nn_err_lp}
     Suppose that $h< \min\left\{ \frac{1}{4G_p}, \frac{1}{G_2}\right\}$. Under Assumption \ref{asp:error_lp}, \ref{asp:smooth}, for any $t\in [kh, (k+1)h]$,
     \begin{equation*}
        \begin{aligned}
             &\qquad \expect_{\mu_t} \bv\bv \nabla g^*(\nabla\log \frac{\pi}{\mu_{kh}}) - \expect [v_k(X_{kh})| X_t=\cdot] \bv\bv_p^p \\
             &\qquad\qquad \leq \frac{2^{p-1}(1-hG_2)^{-d}}{1-(4G_ph)^p}\varepsilon_k + \frac{2^{p-1}(4G_ph)^p}{1-(4G_ph)^p} \left( (1+c_1)\expect_{\mu_t}\bv\bv \nabla\log\frac{\pi}{\mu_t}\bv\bv_q^q + c_2 \left(M_p(1-hG_2)^{-1}dh\right)^q\right),
        \end{aligned}
     \end{equation*}
     where $c_1, c_2$ are defined in \eqref{eq:def_c1c2}.
\end{lemma}

\begin{proof}
    By Jensen's inequality, 
    \begin{equation}\label{eq:split}
        \begin{aligned}
             &\expect_{\mu_t} \bv\bv \nabla g^*(\nabla\log \frac{\pi}{\mu_{kh}}) - \expect [v_k(X_{kh})| X_t=\cdot] \bv\bv_p^p \\
             &\qquad\qquad \leq 2^{p-1} \left\{\expect_{\mu_t} \bv\bv \nabla g^*(\nabla\log \frac{\pi}{\mu_{kh}}) - v_k\bv\bv_p^p + \expect_{\mu_t} \bv\bv v_k - \expect [v_k(X_{kh})| X_t=\cdot] \bv\bv_p^p\right\} \\
             &\qquad\qquad \leq 2^{p-1} \left\{\expect_{\mu_t} \bv\bv \nabla g^*(\nabla\log \frac{\pi}{\mu_{kh}}) - v_k\bv\bv_p^p + \expect \bv\bv v_k(X_t) - v_k(X_{kh}) \bv\bv_p^p\right\}.
        \end{aligned}
    \end{equation}
    Under Assumption \ref{asp:smooth} and Jensen's inequality,
    \begin{equation*}
        \begin{aligned}
            &\expect \bv\bv v_k(X_t) - v_k(X_{kh}) \bv\bv_p^p \\
            &\qquad\qquad \leq (G_p)^p\ \expect \|X_t-X_{kh}\|_p^p \\
            &\qquad\qquad \leq (G_ph)^p\ \expect \bv\bv v_k(X_{kh})\bv\bv_p^p \\
            &\qquad\qquad \leq (G_ph)^p 4^{p-1} \left\{ \expect \bv\bv v_k(X_{kh}) - v_k(X_t)\bv\bv_p^p + \expect \bv\bv v_k(X_t) - \nabla g^*\left(\nabla\log \frac{\pi}{\mu_{kh}}(X_t)\right)\bv\bv_p^p \right.\\
            &\qquad\qquad\qquad \left. +\ \expect \bv\bv \nabla g^*\left(\nabla\log \frac{\pi}{\mu_{kh}}(X_t)\right) - \nabla g^*\left(\nabla\log \frac{\pi}{\mu_t}(X_t)\right)\bv\bv_p^p + \expect \bv\bv \nabla g^*\left(\nabla\log \frac{\pi}{\mu_t}(X_t)\right)\bv\bv_p^p  \right\}.
        \end{aligned}
    \end{equation*}
    Rearrange the above inequality and thus   \begin{equation}\label{eq:split2}
        \begin{aligned}
            &\expect \bv\bv v_k(X_t) - v_k(X_{kh}) \bv\bv_p^p \\
            &\qquad\qquad \leq \left(1-(4G_ph)^p\right)^{-1}(4G_ph)^p \left\{ \expect_{\mu_t} \bv\bv v_k - \nabla g^*(\nabla\log \frac{\pi}{\mu_{kh}})\bv\bv_p^p \right.\\
            &\qquad\qquad\qquad \left. +\ \expect_{\mu_t} \bv\bv \nabla g^*\nabla\log \frac{\pi}{\mu_{kh}}) - \nabla g^*(\nabla\log \frac{\pi}{\mu_t})\bv\bv_p^p + \expect_{\mu_t} \bv\bv \nabla g^*(\nabla\log \frac{\pi}{\mu_t})\bv\bv_p^p\right\}.
        \end{aligned}
    \end{equation}
    Again note that $\frac{\mu_t}{\mu_{kh}}=\det\left(I_d+(t-kh)\nabla v_k\right)^{-1}$, so
    \begin{equation*}
        \sup_{x}\frac{\mu_t}{\mu_{kh}}(x) \leq \left(\|(I_d+(t-kh)\nabla v_k(x))^{-1}\|_2 \right)^{d} \leq (1-hG_2)^{-d}.
    \end{equation*}
    Then by Assumption \ref{asp:error_lp}, 
    \begin{equation}\label{eq:err_t}
        \expect_{\mu_t} \bv\bv v_k - \nabla g^*(\nabla\log \frac{\pi}{\mu_{kh}})\bv\bv_p^p \leq (1-hG_2)^{-d}\expect_{\mu_{kh}} \bv\bv v_k - \nabla g^*(\nabla\log \frac{\pi}{\mu_{kh}})\bv\bv_p^p \leq (1-hG_2)^{-d}\varepsilon_k.
    \end{equation}
    Combining Lemma \ref{lem:discrete_err_lp} with \eqref{eq:split2},\eqref{eq:err_t} and plugging them into \eqref{eq:split}, we finish the proof.    
\end{proof}

\begin{lemma}\label{lem:nn_err_l2}
    Suppose that $h< \frac{1}{4G_2}$. Under Assumption \ref{asp:error_l2}, \ref{asp:smooth}, for any $t\in [kh, (k+1)h]$,
    \begin{equation*}
        \begin{aligned}
             &\qquad \expect_{\mu_t} \bv\bv \nabla g^*(\nabla\log \frac{\pi}{\mu_{kh}}) - \expect [v_k(X_{kh})| X_t=\cdot] \bv\bv_2^2 \\
             &\qquad\qquad \leq \frac{2(1-hG_2)^{-d}}{1-(4G_2h)^2}\varepsilon_k + \frac{2(4G_2h)^2\beta^2}{1-(4G_2h)^2} \left(\expect_{\mu_t}\bv\bv \nabla\log\frac{\pi}{\mu_t}\bv\bv_2^2 + \left(M_2(1-hG_2)^{-1}dh\right)^2\right).
        \end{aligned}
     \end{equation*}
\end{lemma}

\begin{proof}
    The procedure is exactly the same with Lemma \ref{lem:nn_err_lp}. The only difference appears when applying Lemma \ref{lem:discrete_err_l2} instead of Lemma \ref{lem:discrete_err_lp} in the last step.
\end{proof}

\begin{lemma}\label{lem:one_step_lp}
    Suppose that $h\leq \min\left\{\frac{1}{36G_p}, \frac{1-2^{-\frac{1}{q}}}{G_2}\right\}$. Under Assumption \ref{asp:error_lp}, \ref{asp:smooth}, for any $t\in (kh, (k+1)h)$,
    \begin{equation*}
        \partial_t \KL(\mu_t \| \pi) \leq -\frac{1}{12}\expect_{\mu_t} \bv\bv \nabla\log\frac{\pi}{\mu_t}\bv\bv_q^q + A_1(M_pdh)^q + A_2(1-hG_2)^{-d}\varepsilon_k,
    \end{equation*}
    where $A_1, A_2$ are constants only depending on $q$:
    \begin{equation}\label{eq:def_A1A2}
        (A_1, A_2) = \left\{
        \begin{array}{ll}
             (\frac{2^{p+2}}{p},\ 2^p) & \text{if}\ q\leq 2, \\
             (\frac{7c_2}{p},\ \frac{3}{p}) & \text{otherwise}.
        \end{array} \right.
    \end{equation}
    Here $c_2$ is defined in \eqref{eq:def_c1c2}.
\end{lemma}

\begin{proof}
    By Lemma \ref{lem:dKL} and Young's inequality, for any $\lambda_1, \lambda_2 > 0$,
    \begin{equation*}
        \begin{aligned}
            \partial_t \KL(\mu_t \| \pi) 
            &= -\expect_{\mu_t} \left\langle \nabla\log\frac{\pi}{\mu_t}, \nabla g^*(\nabla\log\frac{\pi}{\mu_t})\right\rangle \\
            &\quad + \expect_{\mu_t} \left\langle \nabla\log\frac{\pi}{\mu_t}, \nabla g^*(\nabla\log\frac{\pi}{\mu_t}) - \nabla g^*(\nabla\log\frac{\pi}{\mu_{kh}})\right\rangle \\
            &\quad + \expect_{\mu_t} \left\langle \nabla\log\frac{\pi}{\mu_t}, \nabla g^*(\nabla\log\frac{\pi}{\mu_{kh}}) - \expect[v_k(X_{kh})|X_t=\cdot] \right\rangle \\
            &\leq -(1-\frac{1}{q}\lambda_1^q-\frac{1}{q}\lambda_2^q) \expect_{\mu_t} \bv\bv \nabla\log\frac{\pi}{\mu_t}\bv\bv_q^q \\
            &\quad + \frac{1}{p}\lambda_1^{-p} \expect_{\mu_t} \bv\bv \nabla g^*(\nabla\log\frac{\pi}{\mu_t}) - \nabla g^*(\nabla\log\frac{\pi}{\mu_{kh}}) \bv\bv_p^p\\
            &\quad + \frac{1}{p}\lambda_2^{-p} \expect_{\mu_t} \bv\bv \nabla g^*(\nabla\log\frac{\pi}{\mu_{kh}}) - \expect[v_k(X_{kh})|X_t=\cdot] \bv\bv_p^p.
        \end{aligned}
    \end{equation*}
    Then we apply Lemma \ref{lem:discrete_err_lp} and Lemma \ref{lem:nn_err_lp},
    \begin{equation*}
        \begin{aligned}
            \partial_t \KL(\mu_t \| \pi) 
            &\leq -\left(1-\frac{1}{q}\lambda_1^q-\frac{1}{p}c_1\lambda_1^{-p}-\frac{1}{q}\lambda_2^q - \frac{1}{p}\lambda_2^{-p}\frac{2^{p-1}(4G_ph)^p}{1-(4G_ph)^p}(1+c_1)\right) \expect_{\mu_t} \bv\bv \nabla\log\frac{\pi}{\mu_t}\bv\bv_q^q \\
            &\quad + \frac{c_2}{p}\left(\lambda_1^{-p} + \lambda_2^{-p}\frac{2^{p-1}(4G_ph)^p}{1-(4G_ph)^p}\right) \left( M_p(1-hG_2)^{-1}dh\right)^q + \frac{\lambda_2^{-p}}{p} \frac{2^{p-1}(1-hG_2)^{-d}}{1-(4G_ph)^p}\varepsilon_k.
        \end{aligned}
    \end{equation*}
    
    If $q>2$ so that $c_1 = 3^{-p}<1$, take $\lambda_1 = c_1^{\frac{1}{p+q}}, \lambda_2 = 1$. Note that for $h\leq \frac{1}{36G_p}$, $\frac{2^{p-1}(4G_ph)^p}{1-(4G_ph)^p}(1+c_1)\leq \frac{4G_ph}{1-4G_ph} \cdot \frac{4}{3}\leq \frac{1}{6}$. And thus,
    \begin{equation*}
        \begin{aligned}
            \partial_t \KL(\mu_t \| \pi) 
            &\leq -\left(\frac{2}{3}-\frac{1}{q}-\frac{1}{p}\cdot\frac{1}{6}\right) \expect_{\mu_t} \bv\bv \nabla\log\frac{\pi}{\mu_t}\bv\bv_q^q + \frac{7c_2}{p}(M_pdh)^q + \frac{3}{p}(1-hG_2)^{-d}\varepsilon_k \\
            &\leq -(\frac{5}{6p}-\frac{1}{3}) \expect_{\mu_t} \bv\bv \nabla\log\frac{\pi}{\mu_t}\bv\bv_q^q + \frac{7c_2}{p}(M_pdh)^q + \frac{3}{p}(1-hG_2)^{-d}\varepsilon_k \\
            &\leq -\frac{1}{12} \expect_{\mu_t} \bv\bv \nabla\log\frac{\pi}{\mu_t}\bv\bv_q^q + \frac{7c_2}{p}(M_pdh)^q + \frac{3}{p}(1-hG_2)^{-d}\varepsilon_k.
        \end{aligned}
    \end{equation*}

    If $q\leq 2$ so that $c_1=0$, take $\lambda_1 = \lambda_2 = (\frac{q}{3})^{\frac{1}{q}}$. Note that for $h\leq \frac{1}{36G_p}$, $\frac{2^{p-1}(4G_ph)^p}{1-(4G_ph)^p}(1+c_1)\leq \frac{(8G_ph)^2}{2(1-4G_ph)} \leq \frac{12}{5}G_ph \leq \frac{1}{15}$. And thus,
    \begin{equation*}
        \begin{aligned}
            \partial_t \KL(\mu_t \| \pi) 
            &\leq -\left(\frac{1}{3}-\frac{4}{3p}(\frac{3}{q})^{\frac{q}{p}}G_ph\right) \expect_{\mu_t} \bv\bv \nabla\log\frac{\pi}{\mu_t}\bv\bv_q^q + \frac{c_2q}{p} (\frac{3}{q})^q (M_pdh)^q + \frac{2^p}{p}(\frac{3}{p})^{\frac{q}{p}}(1-hG_2)^{-d}\varepsilon_k \\
            &\leq -\left(\frac{1}{3}-\frac{4}{p}G_ph\right) \expect_{\mu_t} \bv\bv \nabla\log\frac{\pi}{\mu_t}\bv\bv_q^q + \frac{2^{p+2}}{p}(M_pdh)^q + 2^p(1-hG_2)^{-d}\varepsilon_k \\
            &\leq -\frac{1}{6} \expect_{\mu_t} \bv\bv \nabla\log\frac{\pi}{\mu_t}\bv\bv_q^q + \frac{2^{p+2}}{p}(M_pdh)^q + 2^p(1-hG_2)^{-d}\varepsilon_k.
        \end{aligned}
    \end{equation*}
    Therefore, define $A_1,A_2$ as in \eqref{eq:def_A1A2} and we finish the proof.
\end{proof}

\begin{lemma}\label{lem:one_step_l2}
    Suppose that $h\leq \frac{1}{4G_2\sqrt{12\kappa^2+1}}$, where $\kappa:=\frac{\beta}{\alpha}\geq 1$. Under Assumption \ref{asp:error_l2}, \ref{asp:smooth}, for any $t\in (kh, (k+1)h)$,
    \begin{equation*}
       \partial_t \KL(\mu_t \| \pi) \leq -\frac{\alpha}{6} \expect_{\mu_t} \bv\bv \nabla\log\frac{\pi}{\mu_t}\bv\bv_2^2 + \frac{3}{\alpha} (\beta M_2dh)^2 + \frac{4}{\alpha}(1-hG_2)^{-d}\varepsilon_k.
    \end{equation*}
\end{lemma}

\begin{proof}
    Similar to Lemma \ref{lem:one_step_lp}, under $g^*$ is $\alpha$-strongly convex,
    \begin{equation*}
        \begin{aligned}
            \partial_t \KL(\mu_t \| \pi) 
            &= -\expect_{\mu_t} \left\langle \nabla\log\frac{\pi}{\mu_t}, \nabla g^*(\nabla\log\frac{\pi}{\mu_t})\right\rangle \\
            &\quad + \expect_{\mu_t} \left\langle \nabla\log\frac{\pi}{\mu_t}, \nabla g^*(\nabla\log\frac{\pi}{\mu_t}) - \nabla g^*(\nabla\log\frac{\pi}{\mu_{kh}})\right\rangle \\
            &\quad + \expect_{\mu_t} \left\langle \nabla\log\frac{\pi}{\mu_t}, \nabla g^*(\nabla\log\frac{\pi}{\mu_{kh}}) - \expect[v_k(X_{kh})|X_t=\cdot] \right\rangle \\
            &\leq -(\alpha-\frac{1}{2}\lambda_1^2-\frac{1}{2}\lambda_2^2) \expect_{\mu_t} \bv\bv \nabla\log\frac{\pi}{\mu_t}\bv\bv_2^2 \\
            &\quad + \frac{1}{2}\lambda_1^{-2} \expect_{\mu_t} \bv\bv \nabla g^*(\nabla\log\frac{\pi}{\mu_t}) - \nabla g^*(\nabla\log\frac{\pi}{\mu_{kh}}) \bv\bv_2^2\\
            &\quad + \frac{1}{2}\lambda_2^{-2} \expect_{\mu_t} \bv\bv \nabla g^*(\nabla\log\frac{\pi}{\mu_{kh}}) - \expect[v_k(X_{kh})|X_t=\cdot] \bv\bv_2^2.
        \end{aligned}
    \end{equation*}
    Then we apply Lemma \ref{lem:discrete_err_l2} and Lemma \ref{lem:nn_err_l2},
    \begin{equation*}
        \begin{aligned}
            \partial_t \KL(\mu_t \| \pi) 
            &\leq -\left(\alpha-\frac{1}{2}\lambda_1^2-\frac{1}{2}\lambda_2^2 - \lambda_2^{-2}\frac{(4G_2h)^2\beta^2}{1-(4G_2h)^2}\right) \expect_{\mu_t} \bv\bv \nabla\log\frac{\pi}{\mu_t}\bv\bv_2^2 \\
            &\quad + \frac{1}{2}\left(\lambda_1^{-2} + \lambda_2^{-2}\frac{2(4G_2h)^2}{1-(4G_2h)^2}\right) \left(\beta M_2(1-hG_2)^{-1}dh\right)^2 + \frac{\lambda_2^{-2}}{2} \frac{2(1-hG_2)^{-d}}{1-(4G_2h)^2}\varepsilon_k.
        \end{aligned}
    \end{equation*}
    Take $\lambda_1=\lambda_2=\sqrt{\frac{\alpha}{2}}$. Note that for $4G_2h\leq \frac{1}{\sqrt{12\kappa^2+1}}$, we have $\frac{(4G_2h)^2}{1-(4G_2h)^2}\leq \frac{1}{12\kappa^2}$. And thus,
    \begin{equation*}
        \partial_t \KL(\mu_t \| \pi) \leq -\frac{\alpha}{6} \expect_{\mu_t} \bv\bv \nabla\log\frac{\pi}{\mu_t}\bv\bv_2^2 + \frac{3}{\alpha} (\beta M_2dh)^2 + \frac{4}{\alpha}(1-hG_2)^{-d}\varepsilon_k.
    \end{equation*}
\end{proof}

\subsection{Proof of Main Results}

\begin{thm}\label{thm:lp}
    Under Assumption \ref{asp:error_lp}, \ref{asp:smooth}, for any step size $h\leq \min\left\{\frac{1}{36G_p}, \frac{1-2^{-\frac{1}{q}}}{G_2}\right\}$, it holds that
    \begin{equation*}
        \frac{1}{Nh}\int_0^{Nh} \expect_{\mu_t} \bv\bv \nabla\log\frac{\pi}{\mu_t}\bv\bv_q^q dt \leq 12\left(\frac{\KL(\mu_0\|\pi)}{Nh} + A_1(M_pdh)^q + A_2(1-hG_2)^{-d}\frac{\sum_{k=0}^{N-1}\varepsilon_k}{N}\right),
    \end{equation*}
    where $A_1, A_2$ defined in \eqref{eq:def_A1A2} are constants that only depend on $q$. 
    
    Additionally, if $\KL(\mu_0\|\pi)\leq K_0$, then for $N\gtrsim \frac{K_0\left(G_p\vee (qG_2)\right)^{q+1}}{qA_1(M_pd)^q}$, we can choose $h\asymp (\frac{K_0}{qA_1(M_pd)^qN})^{\frac{1}{q+1}}\wedge\frac{1}{dG_2}$. The following bound holds:
    \begin{equation*}
         \expect_{\bar{\mu}_{Nh}} \bv\bv \nabla\log\frac{\pi}{\bar{\mu}_{Nh}}\bv\bv_q^q = \Tilde{\mathcal{O}}\left((\frac{M_pK_0d}{N})^{\frac{q}{q+1}} + \frac{G_2K_0d}{N} + \frac{\sum_{k=0}^{N-1}\varepsilon_k}{N}\right).
    \end{equation*}
    Here $\Tilde{\mathcal{O}}(\cdot)$ hides all the constant factors that only depend on $q$.
\end{thm}

\begin{proof}
    Under Lemma \ref{lem:one_step_lp}, take integral of both sides from $kh$ to $(k+1)h$ and we obtain
    \begin{equation*}
        \KL(\mu_{(k+1)h}\|\pi) - \KL(\mu_{kh}\|\pi) \leq -\frac{1}{12}\int_{kh}^{(k+1)h} \expect_{\mu_t} \bv\bv \nabla\log\frac{\pi}{\mu_t}\bv\bv_q^q dt + A_1(M_pdh)^qh + A_2(1-hG_2)^{-d}\varepsilon_kh.
    \end{equation*}
    Rearranging it and summing from $0$ to $N-1$,
    \begin{equation}\label{eq:lp_bound}
        \frac{1}{Nh}\int_0^{Nh} \expect_{\mu_t} \bv\bv \nabla\log\frac{\pi}{\mu_t}\bv\bv_q^q dt \leq 12\left(\frac{\KL(\mu_0\|\pi)}{Nh} + A_1(M_pdh)^q + A_2(1-hG_2)^{-d}\frac{\sum_{k=0}^{N-1}\varepsilon_k}{N}\right).
    \end{equation}

    Note that for any convex function $g^*$ on $\reals^d$, $(a,b)\mapsto g^*(a/b)b$ is also convex on $\reals^d\times \reals_+$. Therefore, $\mu\mapsto\expect_{\mu}g^*(\nabla\log\frac{\pi}{\mu})$ is convex in the classical sense on the space of probability measures. And thus
    \begin{equation}
        \expect_{\bar{\mu}_{Nh}} \bv\bv \nabla\log\frac{\pi}{\bar{\mu}_{Nh}}\bv\bv_q^q\leq \frac{1}{Nh}\int_0^{Nh} \expect_{\mu_t} \bv\bv \nabla\log\frac{\pi}{\mu_t}\bv\bv_q^q dt.
    \end{equation}
    We finish the proof by plugging the step size $h$ in \eqref{eq:lp_bound} and hiding all the constants that only depend on $q$. 
\end{proof}

\begin{thm}\label{thm:l2}
    Under Assumption \ref{asp:error_l2}, \ref{asp:smooth}, for any step size $h\leq \frac{1}{4G_2\sqrt{12\kappa^2+1}}$, where $\kappa:=\frac{\beta}{\alpha}\geq 1$, it holds that
    \begin{equation*}
        \frac{1}{Nh}\int_0^{Nh} \expect_{\mu_t} \bv\bv \nabla\log\frac{\pi}{\mu_t}\bv\bv_2^2 dt \leq \frac{6}{\alpha}\left(\frac{\KL(\mu_0\|\pi)}{Nh} + \frac{3}{\alpha}\beta^2M_2^2d^2h^2 + \frac{4}{\alpha}(1-hG_2)^{-d}\frac{\sum_{k=0}^{N-1}\varepsilon_k}{N}\right).
    \end{equation*}

    For simplicity, assume $\alpha=1$. If and $\KL(\mu_0\|\pi)\leq K_0$, then we can choose $h\asymp (\frac{K_0}{(\kappa M_2d)^2N})^{\frac{1}{3}}\wedge\frac{1}{dG_2}\wedge\frac{1}{\kappa G_2}$. The following bound holds:
    \begin{equation*}
         \expect_{\bar{\mu}_{Nh}} \bv\bv \nabla\log\frac{\pi}{\bar{\mu}_{Nh}}\bv\bv_2^2 = \mathcal{O}\left((\frac{\kappa M_2K_0d}{N})^{\frac{2}{3}} + \frac{G_2K_0(d+\kappa)}{N} + \frac{\sum_{k=0}^{N-1}\varepsilon_k}{N}\right).
    \end{equation*}
\end{thm}

\begin{proof}
    Under Lemma \ref{lem:one_step_l2}, take integral of both sides from $kh$ to $(k+1)h$ and we obtain
    \begin{equation*}
        \KL(\mu_{(k+1)h}\|\pi) - \KL(\mu_{kh}\|\pi) \leq -\frac{\alpha}{6}\int_{kh}^{(k+1)h} \expect_{\mu_t} \bv\bv \nabla\log\frac{\pi}{\mu_t}\bv\bv_2^2 dt + \frac{3}{\alpha}\beta^2M_2^2d^2h^3 + \frac{4}{\alpha}(1-hG_2)^{-d}\varepsilon_kh.
    \end{equation*}
    Rearranging it and summing from $0$ to $N-1$,
    \begin{equation*}
        \frac{1}{Nh}\int_0^{Nh} \expect_{\mu_t} \bv\bv \nabla\log\frac{\pi}{\mu_t}\bv\bv_2^2 dt \leq \frac{6}{\alpha}\left(\frac{\KL(\mu_0\|\pi)}{Nh} + \frac{3}{\alpha}\beta^2M_2^2d^2h^2 + \frac{4}{\alpha}(1-hG_2)^{-d}\frac{\sum_{k=0}^{N-1}\varepsilon_k}{N}\right).
    \end{equation*}
    
    The remaining part is similar to the proof of Theorem \ref{thm:lp}.
\end{proof}

\subsection{Discussions}\label{app_subsec:kl}
The convergence of score divergence only guarantees that the particle distribution gets the local structure of $\pi$ correct \citep{Balasubramanian2022TowardsAT}. To obtain a stronger convergence guarantee, we still need isoperimetry condition of target distribution. We start with $L_2$-GF.
\begin{thm}\label{thm:LSI}
    If we additionally assume that $\pi$ satisfies \textit{log-Sobolev inequality} with constant $\lambda$, \ie
    \begin{equation*}
        \text{Ent}_\pi(f^2)\leq \frac{2}{\lambda} \expect_{\pi} [\|\nabla f\|_2^2],\ \text{for all smooth}\ f:\reals^d\rightarrow\reals.
    \end{equation*}
    then under the same conditions of Theorem \ref{thm:l2} with $\alpha=\beta=1, \varepsilon_k\leq \epsilon$, it holds that
    \begin{equation*}
         \KL(\mu_{Nh}\|\pi)\leq \exp(-\frac{\lambda Nh}{3}) \KL(\mu_{0}\|\pi) + 3\left(3(M_2dh)^2 + 4(1-hG_2)^{-d}\epsilon\right)\lambda^{-1}.
    \end{equation*}
    In particular, if $\epsilon\lesssim (\frac{M_2}{G_2})^2$, we take $h\asymp \frac{\sqrt{\epsilon}}{M_2d}$ and then we obtain the guarantee $\KL(\mu_{Nh}\|\pi)\lesssim  \lambda^{-1}\epsilon$ after 
    \begin{equation*}
        N = \mathcal{O}\left(\frac{M_2d}{\lambda\sqrt{\epsilon}}\log\frac{\lambda\KL(\mu_0\|\pi)}{\epsilon}\right)\quad\quad \text{iterations}.
    \end{equation*}
\end{thm}

\begin{rmk}
    We match the SOTA rate of LMC under \textit{log-Sobolev} inequality, Hessian smoothness and dissipativity assumption \citep{mou2022improved}. The assumption on smoothness of target Hessian is known to accelerate convergence rate \citep{dalalyan2019user}. But here we do not assume the smoothness of $\log\pi$ explicitly and thus our method can tackle more complex distributions.
\end{rmk}

\begin{proof}
    For $t\in(kh, (k+1)h)$, we apply Lemma \ref{lem:one_step_l2} and \textit{log-Sobolev inequality} and thus
    \begin{equation}
        \partial_t \KL(\mu_t \| \pi) \leq -\frac{\lambda}{3} \KL(\mu_t \| \pi) + 3(M_2dh)^2 + 4(1-hG_2)^{-d}\epsilon.
    \end{equation}
    By Gronwall's inequality,
    \begin{equation}
        \KL(\mu_{(k+1)h}\|\pi)\leq e^{-\lambda h/3} \KL(\mu_{kh}\|\pi) + 3\lambda^{-1}\left(3(M_2dh)^2 + 4(1-hG_2)^{-d}\epsilon\right)(1-e^{-\lambda h/3}).
    \end{equation}
    Iterating the recursive bound,
    \begin{equation}
        \KL(\mu_{Nh}\|\pi)\leq \exp(-\frac{\lambda Nh}{3}) \KL(\mu_{0}\|\pi) + 3\left(3(M_2dh)^2 + 4(1-hG_2)^{-d}\epsilon\right)\lambda^{-1}.
    \end{equation}
\end{proof}

To further interpret our results with GWG under $W_p$ metric, we assume that the target distribution satisfies \textit{modified log-Sobolev inequality}, which has been considered in many classical works \citep{adamczak2017moment, bobkov2000brunn, barthe2008modified}.

\begin{definition}[ \textit{modified log-Sobolev inequality}]
    For $q>1$, we say $\pi$ satisfies the \textit{modified log-Sobolev inequality} mLSI($q, \lambda_q$) if the following holds:
    \begin{equation*}
        \text{Ent}_\pi(|f|^q)\leq \frac{q^{q-1}}{\lambda_q} \expect_{\pi} [\|\nabla f\|_q^q],\ \text{for all smooth}\ f:\reals^d\rightarrow\reals.
    \end{equation*}
\end{definition}

Note that mLSI($2, \lambda_2$) reduces to the conventional \textit{log-Sobolev inequality} with constant $\lambda_2$. As a direct corollary of this inequality, for any distribution $\mu$, we take $f=(\frac{\mu}{\pi})^{1/q}$ and thus
\begin{equation}\label{eq:coro_mLSI}
    \KL(\mu\|\pi) \leq \frac{1}{q\lambda_q} \expect_{\mu} \bv\bv\nabla\log\frac{\pi}{\mu} \bv\bv_q^q.
\end{equation}

\begin{thm}\label{thm:mLSI}
    If we additionally assume that $\pi$ satisfies \eqref{eq:coro_mLSI},
    then under the same conditions of Theorem \ref{thm:lp} with $\varepsilon_k\leq \epsilon$, it holds that
    \begin{equation*}
        \KL(\mu_{Nh}\|\pi)\leq \exp(-\frac{q\lambda_q Nh}{12}) \KL(\mu_{0}\|\pi) + 12\left(A_1(M_pdh)^q + A_2(1-hG_2)^{-d}\epsilon\right)(q\lambda_q)^{-1}.
    \end{equation*}
    In particular, if $\epsilon\lesssim \min\left\{(\frac{M_p}{G_2})^q, (\frac{dM_p}{qG_2})^q\right\}$, we take $h\asymp \frac{\epsilon^{1/q}}{M_pd}$ and then we obtain the guarantee $\KL(\mu_{Nh}\|\pi)\lesssim \lambda_q^{-1}\epsilon$ after 
    \begin{equation*}
        N = \Tilde{\mathcal{O}}\left(\frac{M_pd}{\lambda_q\epsilon^{1/q}}\log\frac{\lambda_q\KL(\mu_0\|\pi)}{\epsilon}\right)\quad\quad \text{iterations}.
    \end{equation*}
    Here $\Tilde{\mathcal{O}}(\cdot)$ hides all the constant factors that only depend on $q$.
\end{thm}

\begin{proof}
    For $t\in(kh, (k+1)h)$, we apply Lemma \ref{lem:one_step_lp} and \eqref{eq:coro_mLSI} and thus
    \begin{equation}
        \partial_t \KL(\mu_t \| \pi) \leq -\frac{q\lambda_q}{12} \KL(\mu_t \| \pi) + A_1(M_pdh)^q + A_2(1-hG_2)^{-d}\epsilon.
    \end{equation}
    By Gronwall's inequality,
    \begin{equation}
        \KL(\mu_{(k+1)h}\|\pi)\leq e^{-q\lambda_q h/12} \KL(\mu_{kh}\|\pi) + 12(q\lambda_q)^{-1}\left(A_1(M_pdh)^q + A_2(1-hG_2)^{-d}\epsilon\right)(1-e^{-q\lambda_q h/12}).
    \end{equation}
    Iterating the recursive bound,
    \begin{equation}
        \KL(\mu_{Nh}\|\pi)\leq \exp(-\frac{q\lambda_q Nh}{12}) \KL(\mu_{0}\|\pi) + 12\left(A_1(M_pdh)^q + A_2(1-hG_2)^{-d}\epsilon\right)(q\lambda_q)^{-1}.
    \end{equation}
\end{proof}

\begin{rmk}
    mLSI($q, \lambda_q$) cannot hold for $q>2$ as mentioned in \citet{barthe2008modified}. However, we only need \eqref{eq:coro_mLSI} to hold for all $\mu=\mu_{t}$ to prove Theorem \ref{thm:mLSI}. Plus, \cite{gentil2005modified, gentil2007modified} replace $\|\cdot\|_q^q$ with $\max\{\|\cdot\|_2^2, \|\cdot\|_q^q\}$ and show that mLSI can hold for a class of distributions in this way. We leave this for future work.
\end{rmk}

\begin{rmk}
    Theorem \ref{thm:mLSI} illustrates how the choice of $q$ will influence the convergence rate of ParVI. On one hand, larger $q$ would reduce the complexity dependence on $\epsilon$. On the other hand, it is generally difficult to predict how $\lambda_q$ will change with $q$ . Besides, large $q$ would also increase the difficulty to train the neural net and obtain a well-estimated direction. Overall, it is challenging to determine the optimal $q$ and thus our adaptive method can present significant advantages.
\end{rmk}

\subsection{Technical Lemmas}
\begin{lemma}\label{lem:logdet}
    For any two matrices $A,B\in \reals^{d\times d}$ with positive eigenvalues, the following holds:
    \begin{equation*}
        |\log\det A - \log\det B| \leq d\|A-B\|_2 \max\{\|A^{-1}\|_2, \|B^{-1}\|_2\}
    \end{equation*}
\end{lemma}

\begin{proof}
   Suppose that the eigenvalues of real matrix $(A-B)B^{-1}$ are $\lambda_1, \overline{\lambda}_1,\cdots, \lambda_k, \overline{\lambda}_k\in \complex,\lambda_{2k+1}, \cdots \lambda_d\in \reals$. Here $\overline{\lambda_j}$ is the complex conjugate of $\lambda_j$. Then it holds that: 
    \begin{equation*}
        \begin{aligned}
            \log\det A - \log\det B
            &= \log\det (I + (A-B)B^{-1}) \\
            &= \log\prod_{j=1}^d(1+\lambda_j) \\
            &\leq \sum_{j=1}^k \log(1+\lambda_j)(1+\overline{\lambda_j}) + \sum_{j=2k+1}^d \log(1+|\lambda_j|) \\
            &\leq \sum_{j=1}^d \log(1+|\lambda_j|) \\
            &\leq d\|(A-B)B^{-1}\|_2
        \end{aligned}
    \end{equation*}
    Similarly, we have $\log\det B - \log\det A\leq d\|(B-A)A^{-1}\|_2$ and thus we finish the proof.
\end{proof}

\begin{lemma}\label{lem:c1c2}
    Define non-negative constants $c_1,c_2$ as:
    \begin{equation*}
        (c_1, c_2) = \left\{
         \begin{array}{ll}
             (0,\ 2^{p-q}) & \text{if}\ q\leq 2, \\
             \left(3^{-p},\ \min\left\{3^{q-p}(\frac{p}{q})^{\frac{q-p}{q-1}}(1-\frac{p}{q})^{\frac{q-p}{p}}, (q-1)^p\left(\frac{(\frac{4}{3})^{\frac{1}{q-1}}}{(\frac{4}{3})^{\frac{1}{q-1}}-1}\right)^{q-p} \right\}\right) & \text{otherwise}.
         \end{array} \right.
    \end{equation*}
    Then for any $a, b\in \reals$, the following inequality holds:
    \begin{equation*}
        \bv \sgn(a)|a|^{q-1} - \sgn(b)|b|^{q-1} \bv^p \leq c_1 |a|^q + c_2 |a-b|^q
    \end{equation*}
\end{lemma}
\begin{proof}
    We shall prove each of the two cases separately.
    
    (1) \textbf{Case} $q\leq 2$.
    
    If $a, b$ have the same sign, we assume they are positive without loss of generality. Then $|a^{q-1} - b^{q-1}| \leq |a-b|^{q-1}$, which implies $|a^{q-1} - b^{q-1}|^p \leq |a-b|^q$.

    If $a, b$ have different signs, we assume $a\geq 0, b<0$ without loss of generality. Then by H{\"o}lder inequality $a^{q-1} + (-b)^{q-1} \leq 2^{2-q} |a-b|^{q-1}$, \ie, $|a^{q-1} - (-b)^{q-1}|^p \leq 2^{p-q}|a-b|^q$.

    (2) \textbf{Case} $q>2$.
    
    If $a, b$ have different signs, we assume $a\geq 0, b<0$ without loss of generality. Then $|a^{q-1} + (-b)^{q-1}| \leq |a-b|^{q-1}$, which implies $|a^{q-1} - b^{q-1}|^p \leq |a-b|^q$.

    If $a, b$ have the same sign, we assume they are positive without loss of generality. Note that this inequality is homogeneous, we can let $a=1$ so that we only need to show for any $b>0$,
    \begin{equation}\label{eq:a=1}
        \bv 1 - b^{q-1} \bv^p \leq c_1 + c_2 |1-b|^q.
    \end{equation}
    
    If $b\leq1$, then by simple calculus,
    \begin{equation*}
        \begin{aligned}
             (1-b^{q-1})^p-c_2(1-b)^q
             &\leq \left[(q-1)(1-b)\right]^p-c_2(1-b)^q \\
             &\leq c_2^{-\frac{p}{q-p}}(q-1)^{\frac{q+p}{q-p}}(\frac{p}{q})^{\frac{p}{q-p}}(1-\frac{p}{q}) \\
             &\leq c_1.
        \end{aligned}
    \end{equation*}
    
    If $1<b\leq \left(1+c_1^{1/p}\right)^{\frac{1}{q-1}}=(\frac{4}{3})^{\frac{1}{q-1}}$, then \eqref{eq:a=1} is trivial.
    
    If $b>(\frac{4}{3})^{\frac{1}{q-1}}$,
    \begin{equation*}
        \begin{aligned}
            (b^{q-1}-1)^p - c_2(b-1)^q 
            &= (b-1)^p [(\frac{b^{q-1}-1}{b-1})^p-c_2(b-1)^{q-p}] \\
            &\leq (b-1)^p \left([(q-1)b^{q-2}]^p-c_2(b-1)^{q-p}\right) \\
            &\leq 0.
        \end{aligned}
    \end{equation*}
    The last inequality is due to $c_2\geq (q-1)^p\left(\frac{(\frac{4}{3})^{\frac{1}{q-1}}}{(\frac{4}{3})^{\frac{1}{q-1}}-1}\right)^{q-p}\geq (q-1)^p\left(\frac{b}{b-1}\right)^{q-p}$.
\end{proof}

\section{Proof of Proposition \ref{prop:ada}}\label{app_sec:prop_ada}
\begin{proof}
    Note that $g(\cdot)=\frac{1}{p}\|\cdot\|_p^p$ and thus $\nabla g^*(\cdot)=|\cdot|^{q-1}\odot sgn(\cdot)$. By Young's inequality,
    \begin{equation}\label{eq:dKL_continuous}
        \begin{aligned}
            \partial_t \KL(\mu_t \| \pi) 
            &= -\expect_{\mu_t} \left\langle \nabla\log\frac{\pi}{\mu_t}, \nabla g^*(\nabla\log\frac{\pi}{\mu_t})\right\rangle \\
            &\quad + \expect_{\mu_t} \left\langle \nabla\log\frac{\pi}{\mu_t}, \nabla g^*(\nabla\log\frac{\pi}{\mu_t}) - f_t\right\rangle \\
            &\leq -(1-\frac{1}{q}\lambda_1^q) \expect_{\mu_t} \bv\bv \nabla\log\frac{\pi}{\mu_t}\bv\bv_q^q \\
            &\quad + \frac{1}{p}\lambda_1^{-p} \expect_{\mu_t} \bv\bv \nabla g^*(\nabla\log\frac{\pi}{\mu_t}) - f_t \bv\bv_p^p.
        \end{aligned}
    \end{equation}
    where $\lambda_1$ could by any positive scalar. Here we set $\lambda_1=1$ and finish the proof.
\end{proof}

\begin{rmk}
    \eqref{eq:A(p)} is not the only way to adaptively choose optimal $p$. In fact, \eqref{eq:dKL_continuous} provides a wide range of methods based on different choices of $\lambda_1$ and thus we can obtain different objectives for $p$. We leave it for future work.
\end{rmk}

\section{Additional Details of Experiments}\label{sec:exp_setup}

\subsection{Gaussian Mixture}
We follow the same setting as \citet{dong2023particlebased}. The marginal probability
of each cluster is 1/10. The number of particles is 1000. For $L_2$-GF, PFG and Ada-GWG, we parameterize $f_w$ as 3-layer neural networks with $\textit{tanh}$ activation function. Each hidden layer has $32$ neurons. The inner loop iteration is 5 and we use SGD optimizer with Nesterov momentum (momentum 0.9) to train $f_w$ with learning rate $\eta$=1e-3. The particle step size is $0.1$.

For PFG, following \citet{dong2023particlebased}, we set the preconditioning matrix $H={\hat{H}}^{\alpha}$, where $\hat{H}$ is the inverse of diagonal variance of particles and $\alpha$ is $ 1.0$.

For Ada-GWG, we set the initial exponent $p_0=2$ and learning rate $\Tilde{\eta}= 2.5$e-7. 

Figure \ref{fig:mog_quant} shows some quantitative comparisons between different algorithms. Here Exp-GF represents GWG with $g(\cdot) = \exp(\|\cdot\|_2^2/(2\sigma^2))-1$. The results are the averaged after 10 random trials. We can observe that Ada-GWG can obtain highly-accurated samples within fewer iterations.

\begin{figure}[H]
    \centering
    \hspace{-3mm}
    \subfigure[JS divergence]
    {\includegraphics[clip,width=0.5\textwidth]{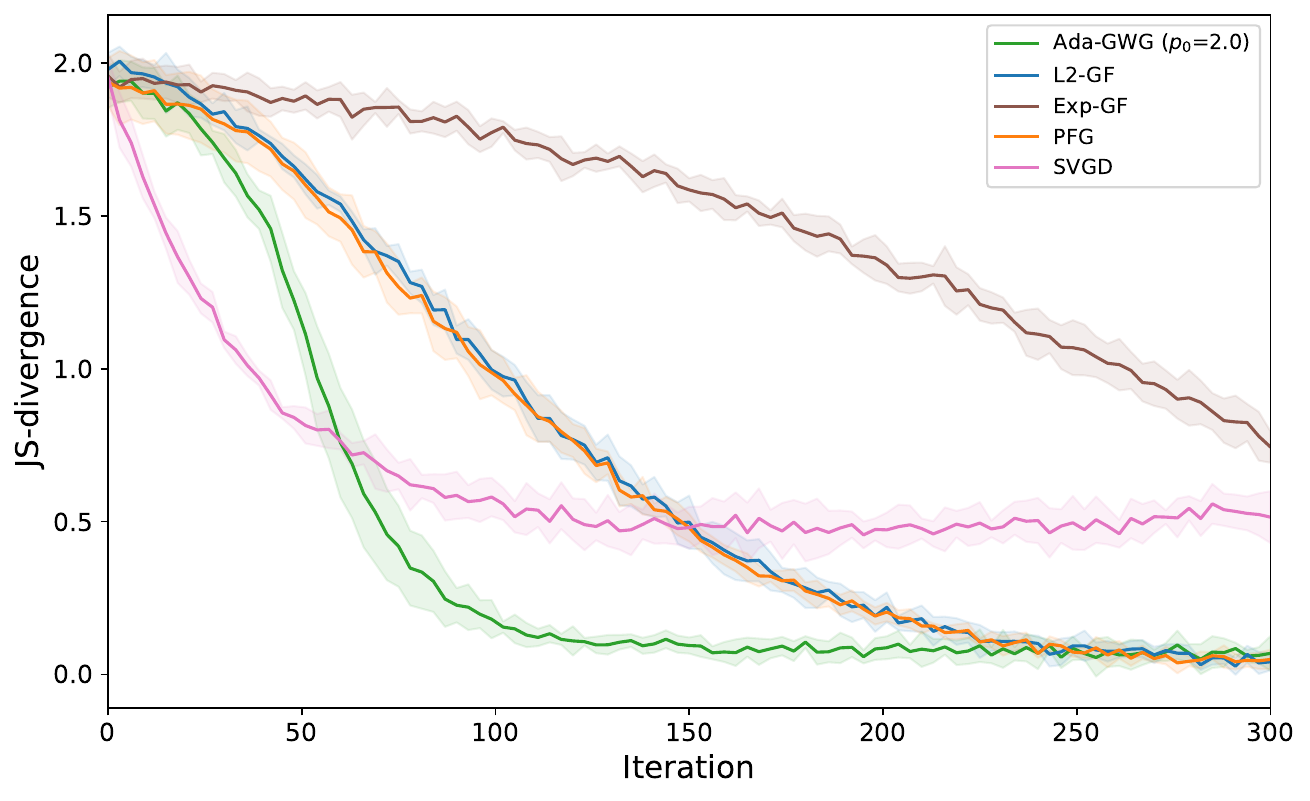}} \hspace{-3mm}
    \subfigure[Evolution of $p$]
    {\includegraphics[clip,width=0.5\textwidth]{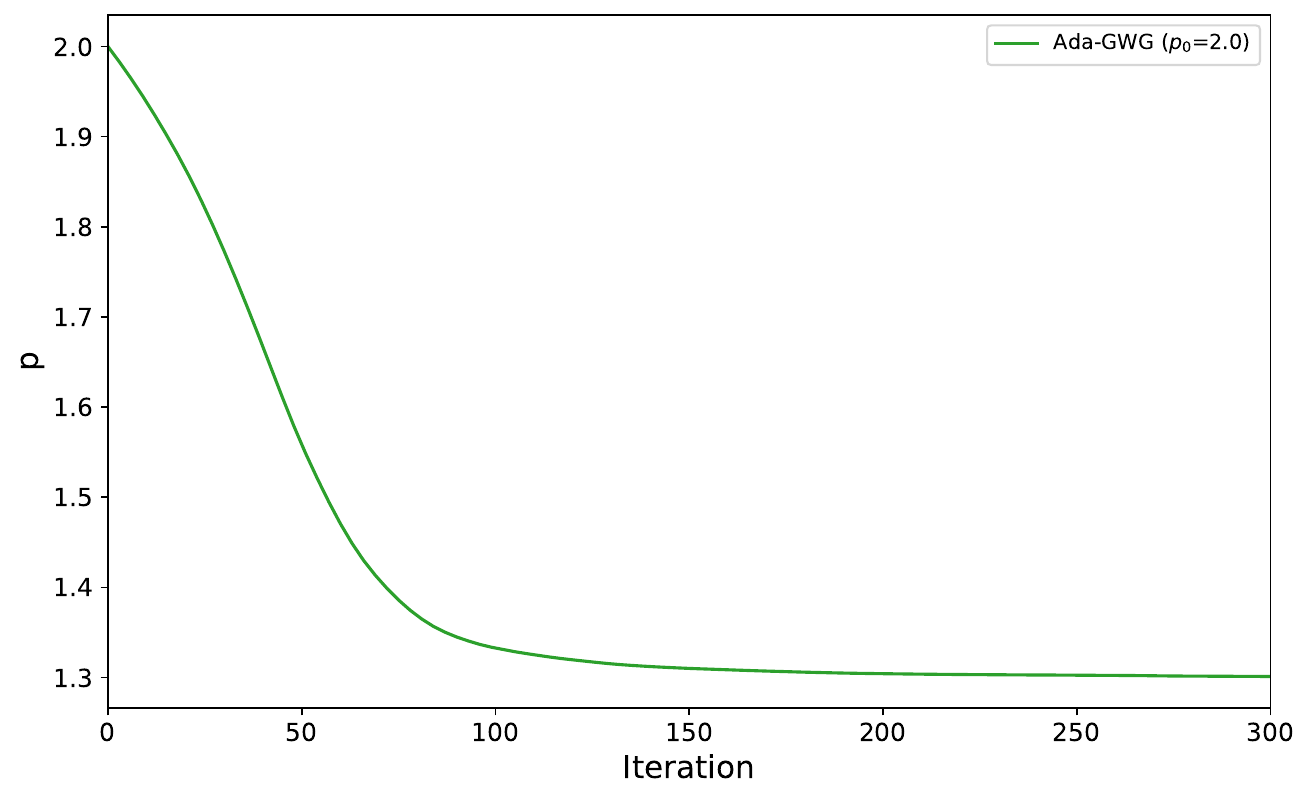}} \hspace{-3mm}
    \caption{Quantitative results in Gaussian Mixture experiment}
    \label{fig:mog_quant}
\end{figure}

\subsection{Monomial Gamma}
On heavy tailed distributions, the number of particles is 1000. For GWG and Ada-GWG, the neural network structure is the same as which in the Gaussian mixture experiment. The inner loop iteration is also 5, but we use Adam optimizer with learning rate $\eta$=1e-3 to train $f_w$ for better stability. The particle step size is 1e-3.

For Ada-GWG, we set the initial exponent $p_0\in \{1.5, 2.0, 2.2\}$ and learning rate $\Tilde{\eta}= 1$.

We run the experiment on 4 random seeds. The average results and the variances are represented in the figure using lines and shades.

\subsection{Conditioned Diffusion}
The procedure to generate the true path is exactly the same as in \citet{detommaso2018stein}. For PFG and Ada-GWG, we parameterize $f_w$ as 3-layer neural networks with $\textit{tanh}$ nonlinearities. Each hidden layer has $200$ neurons. The inner loop $N'$ is selected from $\{1, 5, 10, 15\}$ to get the best performance. $f_w$ is pre-trained for $100$ iterations before particle updates and we use Adam optimizer with learning rate $\eta$=1e-3 to train $f_w$. The particle step size is $3e-3$ for Ada-GWG and PFG. 

For Ada-GWG, we set the initial exponent $p_0=2.2$ and learning rate $\Tilde{\eta}=0.001$. The gradient of $A(p)$ is also clipped within $[-0.1, 0.1]$. 

For PFG, we set the preconditioning matrix $H={\hat{H}}^{\alpha}$, where $\hat{H}$ is the inverse of diagonal variance of particles and $\alpha$ is chosen from $\{0.1, 0.5, 1.0\}$ to obtain the best performance.

For SVGD, we use RBF kernel $\exp(-\frac{\|x-y\|^2}{h})$ where $h$ is the heuristic bandwidth \citep{Liu2016SVGD}. The initial step size is 1e-3 and is adjusted by AdaGrad.

Additionally, we run LMC with step size 1e-4 for $10000$ iterations as the ground truth for posterior distribution.

\subsection{Bayesian Neural Networks}\label{sec:detail-bnn}
Our experiment settings are almost similar to SVGD~\citep{Liu2016SVGD}. 
For the UCI datasets, the datasets are randomly partitioned into 90\% for training and 10\% for testing. 
Then, we further split the training dataset by 10\% to create a validation set for hyperparameter selection as done in~\citep{Liu2016SVGD}. 
For $L_2$-GF and Ada-GWG, we parameterize $f_w$ as 3-layer neural networks. Each hidden layer has $300$ neurons, and we use LeakyReLU as the activation function with a negative slope of 0.1. The inner loop $N'$ is selected from $\{1, 5, 10\}$. We use the Adam optimizer and choose the learning rate from $\{0.001, 0.0001\}$ to train $f_w$.

For Ada-GWG, we choose the initial exponent $p_0$ from $\{3, 4\}$ and set the learning rate $\Tilde{\eta} = 0.0001$. The gradient of $A(p)$ is clipped within [-0.2, 0.2]. We select the step size of particle updates from $\{0.0001, 0.0002, 0.0005, 0.001\}$. For SVGD, we use the RBF kernel as done in~\citep{Liu2016SVGD}.
For SVGD, $L_2$-GF, and Ada-GWG, the iteration number is chosen from $\{2000, 4000\}$ to converge. 
For SGLD, the iteration number is set to 10000 to converge.

\section{Limitations and Future Work}

\paragraph{Estimating Wasserstein gradient by neural networks.} Our formulation leverages the capability of neural networks to estimate the generalized Wasserstein gradient. This approach successfully resolves the problem of kernel design for conventional ParVI methods. However, in high dimensional regime, the design of neural network structure is still important but subtle. Besides, the computation cost is also expensive. We expect more efficient algorithms on training neural works to approximate Wasserstein gradient, \eg \citet{wang2022optimal}.

\paragraph{Better adaptive method.} Our Ada-GWG method is based on the idea of maximizing the decent rate of KL divergence and heavily relies on an accurate estimation of generalized Wasserstein gradient. We update exponent $p$ by simply gradient ascent which may cause severe numerical instability. Although this can be alleviated by clipping, it is still delicate when the target distribution is complex. 

\paragraph{General Young function class.} In this paper, we only consider the function class $\left\{\frac{1}{p}\|\cdot\|_p^p:p>1\right\}$, which is still limited. We expect a more general function class that both have numerical stability and can perfectly capture the information from score function. The characteristics of Young function class may be 
instructive. How to design an adaptive algorithm on a more general class is also challenging and important. We leave this to future research.

\paragraph{Shortness of theoretical analysis.} Although we provide convergence guarantee under weak assumptions, our analysis is still preliminary and we believe these results can be strengthened. There are also other important extensions to consider. For example, our analysis is based on the population loss, which is an asymptotic result based on infinite particles limit. We believe this framework can be also generalized to finite-particle system like SVGD \citep{korba2020non}. Moreover, we only consider Young functions that have the form of $\|\cdot\|_p^p$ or are strongly convex and strongly smooth. We believe that better-designed Young functions may have more advantageous theoretical properties, \eg, Wasserstein Newton flow \citep{Wang2020InformationNF}.

\end{document}